\documentclass[preprint,12pt]{elsarticle}

\usepackage{amssymb}
\usepackage{amsmath}
\usepackage{amsfonts}
\usepackage{amsthm}

\usepackage{graphicx}
\usepackage{subfig}
\usepackage{float}

\usepackage{array}
\usepackage{multirow}
\usepackage{booktabs}
\usepackage{float}
\usepackage{algorithmic}
\usepackage{algorithm}

\usepackage{bm}
\usepackage{textcomp}
\usepackage{url}
\usepackage{hyperref}
\usepackage{color}
\usepackage{mathrsfs}
\usepackage{comment}
\usepackage{rotating}

\usepackage[section]{placeins}
\usepackage{stfloats}

\newtheorem{theorem}{Theorem}
\newtheorem{lemma}{Lemma}
\newtheorem{definition}{Definition}

\newtheorem{prop}{Proposition}

\def\x{\mathbf{x}}
\def\y{\mathbf{y}}

\journal{Pattern Recognition}

\begin{document}

\begin{frontmatter}

\title{NeuronSeek: On Stability and Expressivity of Task-driven Neurons}


\author[cuhk]{Hanyu Pei\fnref{equal}}

\author[cuhk,polyu]{Jing-Xiao Liao\fnref{equal}}

\author[riken]{Qibin Zhao}

\author[hust]{Ting Gao}

\author[polyu2]{Shijun Zhang}

\author[polyu]{Xiaoge Zhang\corref{cor2}}
\ead{xgzhang@polyu.edu.hk}

\author[cuhk]{Feng-Lei Fan\corref{cor2}}
\ead{fenglfan@cityu.edu.hk}

\cortext[cor2]{Co-corresponding authors.}
\fntext[equal]{Co-first authors.}

\affiliation[cuhk]{organization={Frontier of Artificial Networks (FAN) Lab, Department of Data Science, \\ City University of Hong Kong},
            city={Hong Kong},
            country={SAR of China}}

\affiliation[polyu]{organization={Department of Industrial and Systems Engineering, The Hong Kong Polytechnic University},
            city={Hong Kong},
            country={SAR of China}}

\affiliation[riken]{organization={RIKEN Center for Advanced Intelligence Project},
            city={Tokyo},
            country={Japan}}

\affiliation[hust]{organization={Center for Mathematical Sciences, and School of Mathematics and Statistics, \\ Huazhong University of Science and Technology},
            city={Wuhan},
            country={China}}

\affiliation[polyu2]{organization={Department of Applied Mathematics, The Hong Kong Polytechnic University }, 
            city={Hong Kong},
            country={SAR of China}}

\begin{abstract}
Drawing inspiration from the human brain's functional specialization, some recent advances in NeuroAI have ventured into \textit{task-driven neurons}, which tailor neuronal formulations to specific tasks. However, existing prototypes rely heavily on Symbolic Regression (SR) via Genetic Programming (GP), which is plagued by the intrinsic weaknesses of the genetic programming such as the instability in searching for the optimal solution. To transcend these bottlenecks, we propose a paradigm shift from evolutionary heuristics to differentiable optimization. We propose NeuronSeek via Tensor Decomposition (NS-TD), a novel framework that reformulates the discrete search for neuronal structures into a continuous, low-rank tensor optimization problem. This transformation not only guarantees convergence stability but also significantly accelerates the discovery of better aggregation functions. Furthermore, we provide a rigorous theoretical foundation by proving that modifying aggregation functions---while retaining common activation functions---grants the network a ``super-super-expressive'' capability: approximating any continuous function with an arbitrarily small error using a fixed number and magnitude of parameters. Extensive empirical evaluations demonstrate that NS-TD establishes a new state-of-the-art, not only strictly dominating traditional SR-based methods in stability but also  exhibiting competitive performance relative to state-of-the-art machine learning models. The code is available at \url{https://github.com/HanyuPei22/NeuronSeek}.
\end{abstract}


\begin{highlights}
\item We propose a stable neuron discovery paradigm by differentiable tensor optimization.
\item We prove that task-driven neurons possess ``super-super-expressive'' property with fixed number and magnitude of parameters.
\item Our method achieves superior performance and enhanced stability compared to traditional symbolic regression methods and other state-of-the-art baselines across diverse benchmarks.
\end{highlights}


\end{frontmatter}



\section{Introduction}
\label{sec:intro}

Deep learning has achieved significant success in a wide range of fields~\cite{bengio2021machine,lample2020deep}. Credits for these accomplishments largely go to the development of scaling law, wherein basic computational units such as convolutions and attention mechanisms are designed and then scaled per a large-scale architecture. The resultant large network will embrace higher intelligence than human's feature engineering given well-curated big data. Prominent examples include ResNet~\cite{he2016deep}, Transformer~\cite{vaswani2017attention}, and Mamba~\cite{macavaney2021mamba}.
Recently, motivated by neuroscience's paramount contributions to AI, an emerging field called \textit{NeuroAI} has garnered significant attention, which draws upon the principles of biological circuits in the human brain to catalyze the next revolution in AI~\cite{zador2023catalyzing}. The ambitious goal of NeuroAI is rooted in the belief that the human brain, as one of the most intelligent systems, inherently possesses the capacity to address complex challenges in AI development~\cite{fan2023towards,marcus2020next}. Hence, it can serve as a valuable source of inspiration for practitioners, despite the correspondence between biological and artificial neural networks sometimes is implicit.

Following NeuroAI, let us analyze the scaling law through the lens of brain computation. It is seen that the human brain generates complex intellectual behaviors through the cooperative activity of billions of mutually connected neurons with varied morphologies and functionalities~\cite{peng2021morphological}. This suggests that the brain simultaneously benefits from scale at the macroscopic level and neuronal diversity at the microscopic level. The latter is a natural consequence of stem cells' directed programming in order to facilitate the task-specific information processing. Inspired by this observation, extensive research has shown that incorporating neuronal diversity into artificial neural networks can markedly enhance their capabilities \cite{fan2023towards, xu2022quadralib, xu2025dmixnet}. For instance, researchers have transitioned beyond traditional inner-product neurons by integrating diverse nonlinear aggregation functions, such as quadratic neurons \cite{xu2022quadralib}, polynomial neurons \cite{chrysos2020deep}, and dendritic neurons \cite{xu2025dmixnet}.
Such networks have achieved superior performance in tasks like image recognition and generation, underscoring the promise and practicality of innovating a neural network at the level of basic computational units.

Along this direction, driven by the task-specific nature of neurons in our brain, recent research~\cite{fan2024no} introduced a systematic framework for task-driven neurons. Hereafter, we refer to prototyping task-driven neurons as \textit{NeuronSeek}. This approach assumes that there is no single type of neurons that can perform universally well. Therefore, it enables neurons to be tailored for specific tasks through a two-stage process.

\begin{figure}[!htbp]
    \centering
\includegraphics[width=0.65\linewidth]{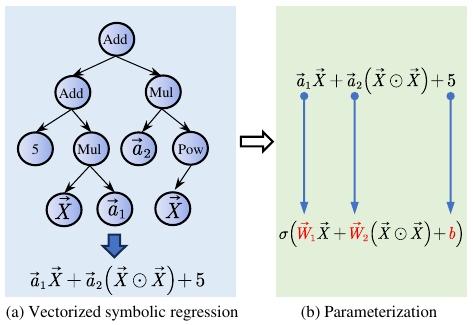}
    \caption{Task-driven neuron based on the vectorized symbolic regression (NS-SR).}
    \label{fig:preliminary}
\end{figure}

As shown in Fig.~\ref{fig:preliminary}, in the first stage, the vectorized symbolic regression (SR)---an extension of symbolic regression~\cite{schmidt2009distilling, bartlett2023exhaustive, NUNEZMOLINA2026104471} that enforces uniform operations for all variables---identifies the optimal homogeneous neuronal formula from data. Unlike traditional linear regression and symbolic regression, the vectorized SR searches for coefficients and homogeneous formulas at the same time, which can facilitate parallel computing. In the second stage, the derived formula serves as the neuron's aggregation function, with learnable coefficients. Notably, activation functions remain unmodified in task-driven neurons. At the neuronal level, task-driven neurons are supposed to have superior performance than those one-size-fits-all units do, as it imbues the prior information pertinent to the task. More favorably, task-driven neurons can still be connected into a network to fully leverage the power of \textit{connectionism}. Hereafter, we refer to task-driven neurons using symbolic regression as NS-SR. NeuronSeek considers in tandem neuronal importance and scale. From design to implementation, it systematically integrates three cornerstone AI paradigms: \textit{behaviorism} (the genetic algorithm runs symbolic regression), \textit{symbolism} (symbolic regression searches a formula from data), and \textit{connectionism} (task-based neurons are connected into a network). Despite this initial success, the following two major challenges remain unresolved in prototyping task-driven neurons:

\begin{itemize}

\item \textbf{Instability of Discrete Heuristic Search.} Existing frameworks rely on Genetic Programming (GP) to navigate the vast space of mathematical expressions. This \textbf{discrete, tree-based heuristic approach} suffers from non-deterministic convergence, particularly in high-dimensional feature spaces. The discovered formulas are often fragile---highly sensitive to random seed initialization, population size, and mutation rates---rendering the resulting neurons inconsistent and difficult to reproduce. This trial-and-error nature fundamentally hurts practitioners' trust on task-driven neurons.

\item \textbf{Necessity of Task-driven Neurons.} In the realm of deep learning theory, it was shown \cite{yarotsky2021elementary, 2021Deep, Shen_2021} that certain activation functions can empower a neural network using a fixed number of neurons to approximate any continuous function with an arbitrarily small error. These functions are termed ``super-expressive'' activation functions. Such a unique and desirable property allows a network to achieve a precise approximation without increasing structural complexity. Since task-based neurons revise aggregation functions instead of activation functions, it is necessary to figure out if the super-expressive property can be achieved via revising the aggregation function while retaining common activation functions?

\end{itemize}

In this study, we address the above two issues satisfactorily. On the one hand, in response to the instability of NS-SR, we utilize a tensor decomposition (TD) method for neuronal discovery, referred to as NS-TD. This approach begins by assuming that the optimal representation of data is encapsulated by a high-order polynomial coupled with trigonometric functions. We construct the basic formulation and apply TD to optimize its low-rank structure and coefficients. Therefore, we reformulate the unstable formula search
problem (NS-SR) into a stable low-rank coefficient optimization task. To enhance robustness, we introduce the sparsity regularization in the decomposition process, automatically eliminating insignificant terms while ensuring the framework to identify the optimal formula. Additionally, rank regularization is employed to derive the simplest possible representation. The improved stability of our method stems from two key factors: i) Unlike conventional symbolic regression which suffers from complex hyperparameter tuning in genetic programming~\cite{makke2023symbolic}, NS-TD renders significantly fewer tunable parameters. ii) It has been proven that TD with rank regularization tends to yield a unique solution~\cite{kolda2009tensor}, thereby effectively addressing the inconsistency issues inherent in symbolic regression methods.

On the other hand, we close the theoretical deficit by showing that task-driven neurons that use common activation such as ReLU can also achieve the super-expressive property. Earlier theories like \cite{yarotsky2021elementary, 2021Deep, Shen_2021} first turn the approximation problem into the point-fitting problem, and then uses the dense trajectory to realize the super-expressive property. The key message is the existence of a one-dimensional dense trajectory to cover the space of interest. In this study, we highlight that dense trajectory can also be achieved by adjusting aggregation functions. Specifically, we integrate task-driven neurons into a discrete dynamical system defined by the layer-wise transformation $\x_N=T(\x_{N-1})$, where $T$ is the mapping performed by one layer of the task-driven network. There exists some initial point $\x_0$ that can be taken by $T$ to the neighborhood of any target point $\x$. To approximate $\x$, our construction does not need to increase the values of parameters in $T$ and the network, and we just compose the transformation module $T$ different times. This means that we realize ``super-super-expressiveness'': not only the number of parameters but also the magnitudes remain fixed. In summary, our contributions are threefold:

\begin{itemize}
    \item \textbf{Paradigm Shift to Differentiable Search.} Our work suggests a paradigm shift in neuronal structure discovery by transitioning from discrete evolutionary heuristics to a differentiable tensor optimization framework. Extensive evaluations demonstrate that this continuous optimization approach significantly surpasses traditional symbolic regression methods, achieving superior fitting capability and search efficiency while guaranteeing convergence stability.

    \item \textbf{Theoretical Guarantee of ``Super-Super-Expressiveness''.} We provide a rigorous proof that task-driven neurons, even when using standard activation functions (e.g., ReLU), possess the ``super-super-expressive'' property. Distinct from prior constructions that only fix the number of parameters, our theory demonstrates that approximation capability holds even when both the \textit{number} and \textit{magnitude} of parameters are fixed, solidifying the theoretical foundation of the field.

    \item \textbf{Superior Performance and Robustness.} Systematic experiments on public datasets and real-world applications demonstrate that NS-TD achieves not only superior performance but also enhanced stability compared to NS-SR and traditional symbolic regression methods in the task of neuronal structure search. Moreover, its predictive performance remains competitive relative to other state-of-the-art baselines.

\end{itemize}

\section{Related Works}

\subsection{Neuronal Diversity and Task-driven Neurons}

Many groundbreaking AI advancements were inspired by the computing of biological neural systems. For instance, the neocognitron~\cite{fukushima2019recent}, a precursor to convolutional neural networks, is exactly a miniature of cortical visual circuits. The human brain is renowned for its extensive array of neurons, which vary significantly in morphology and functionality. This neuronal diversity underpins the brain's intelligent behavior~\cite{deraedt2021statistical}. However, traditional artificial neural networks often prioritize the scaling law that designs a fundamental unit and replicates it across large architectures. This approach underestimates the importance of the diversity of basic computational units. Recently, the idea of introducing neuronal diversity emerged, which mainly modifies the aggregation function instead of the activation function in constructing new neurons. This operation is well-grounded: Once an activation function is monotonic, a neuron's decision boundary is exclusively shaped by its aggregation function. In particular, it replaces the inner product of a neuron with a nonlinear function such as a polynomial. Table~\ref{table:quadratic_neurons} summarizes the recently proposed non-linear neurons. Notably, the complexity of neurons in \cite{zoumpourlis2017non, jiang2020nonlinear,mantini2021cqnn} is of $\mathcal{O}(n^2)$, which is much larger than the conventional neuron, while neurons in \cite{goyal2020improved,bu2021quadratic,xu2022quadralib,fan2018new} favorably enjoy the linear parametric complexity.

\begin{table}[ht]
\centering
\caption{A summary of the recently-proposed neurons. $\sigma(\cdot)$ is the nonlinear activation function. $\odot$ denotes Hadamard product. $\mathbf{W} \in \mathbb{R}^{n \times n}$, $\mathbf{w_i} \in \mathbb{R}^{n \times 1}$, and the bias terms in these neurons are omitted for simplicity.}
\begin{tabular}{lcc}
\hline
\textbf{Works}          & \textbf{Formulations}                                       \\ \hline
Zoumponuris et al. (2017) ~\cite{zoumpourlis2017non} & $\y = \sigma (\x^\top \mathbf{W} \x + \mathbf{w}^\top \x)$                  \\
Fan et al. (2018)~\cite{fan2018new}        & $\y = \sigma((\mathbf{w}_1^\top \x)(\mathbf{w}_2^\top \x) + \mathbf{w}_3^\top (\x \odot \x))$ \\
Jiang et al. (2019)~\cite{jiang2020nonlinear}       & $\y = \sigma(\x^\top  \mathbf{W} \x)$                             \\
Mantini \& Shah (2021)~\cite{mantini2021cqnn}    & $\y = \sigma(\mathbf{w}^\top (\x \odot \x))$                                \\
Goyal et al. (2020)~\cite{goyal2020improved}      & $\y = \sigma(\mathbf{w}^\top (\x \odot \x))$                                \\
Bu \& Karpante (2021)~\cite{bu2021quadratic}    & $\y = \sigma((\mathbf{w}_1^\top \x)(\mathbf{w}_2^\top \x))$                 \\
Xu et al. (2022)~\cite{xu2022quadralib}         & $\y = \sigma((\mathbf{w}_1^\top \x)(\mathbf{w}_2^\top \x) + \mathbf{w}_3^\top \x)$ \\
Fan et al. (2024) \cite{fan2024no} & Task-driven polynomial \\
\hline
\end{tabular}
\label{table:quadratic_neurons}
\end{table}
\newpage
Following the way paved for task-specific neuronal design \cite{fan2024no}, this work constructs neurons by tensor decomposition. It simultaneously resolves theoretical deficits and improves the empirical stability of the original framework, representing a substantial advance in handling high-dimensional learning tasks like image recognition. While Chrysos et al.~\cite{chrysos2021deep, chrysos2022augmenting, chrysos2023regularization} focus on designing polynomial network architectures, our work differs by concentrating on task-specific neuronal construction and more diverse functional forms such as trigonometric functions ($\sin(\x)$, $\cos(\x)$).

\subsection{Symbolic Regression and Deep Learning}

Symbolic regression (SR) is an algorithm that searches for formulas from data~\cite{schmidt2009distilling}. Recent research in SR includes learning partial differential equations (PDEs) from data~\cite{kiyani2023framework}, improving the validity of SR on high-dimensional data \cite{sahoo2018learning}, and invertible SR \cite{tohme2024isr}. Besides, the Kolmogorov-Arnold Network (KAN) is an exemplary work that combines symbolic regression and deep learning~\cite{liu2024kan}. KAN is inspired by the celebrated Kolmogorov-Arnold representation theorem, which states that any continuous multivariate function can be represented as a superposition of several continuous univariate functions. Such a decomposability lays KAN a solid theoretical foundation for function approximation. Specifically, Kolmogorov's theorem asserts that for any vector  $\x=\left[x_{1}, x_{2}, \cdots, x_{d}\right]^{T} \in[0,1]^{d}$, a continuous function $f(\x)$ can be expressed as
\begin{equation}
f(x_1,x_2,\cdots,x_d)=\sum_{i=0}^{2 d} g_{i}\left(\sum_{j=1}^{d} h_{i, j}\left(x_{j}\right)\right),
\end{equation}
where $g_{i}$ and $h_{i,j}$ are continuous univariate functions. Building upon this theorem, KAN is a neural network to exploit this decomposition by learning $g_{i}$ and $h_{i,j}$, which can be regarded as a special symbolic regression. Currently, this KAN-type line has been successfully used in a myriad of domains, such as medical image analysis \cite{li2025u} and scientific discovery \cite{toscano2025pinns}.

Task-based neurons have fundamental differences from KAN in two aspects: i) Our work modifies what a single neuron computes. KAN modifies how the entire graph is connected and parameterized. KANs completely remove the concept of a "neuron" with a fixed activation function. Instead, they place learnable univariate functions on the edges (weights), which sticks to a fixed topology. Thus, our work is of two stages, which is flexible in network structure design, which can fully leverage the power of \textit{connectionism}. ii) Our work explicitly follows a "one-for-one" philosophy. Users must re-run VSR for each new task/dataset to prototype a specific neuron. In contrast, KAN aims to be a universal approximator architecture; users do not redesign the KAN formula for a specific task. Lastly, it is important to note that the concept of task‑based neurons \cite{fan2024no}—using symbolic regression to pre‑design a neuron’s aggregation function—was developed independently and earlier. \cite{fan2024no} was finished and submitted in Nov 2023. Due to the subsequent grant review process, its dissemination was delayed. Hence, the appearance of KAN precedes our arXiv posting, but not our initial conception. We view KAN as an exciting parallel development that, together with our work, underscores the importance of symbolic computing in neural networks.

Moreover, deep learning is also applied to strengthen the performance of SR algorithms. Kamienny et al.~\cite{kamienny2022end} combined the Transformer and symbolic regression to enhance its searching ability. Kim et al.~\cite{kim2021integration} integrated the neural network with symbolic regression, and proposed a framework similar to the equation learner network. SymbolicGPT, a Transformer-based language model for symbolic regression~\cite{valipour2022symbolicgpt}, is also shown to be a competent model compared with the existing models in terms of accuracy and run time.

To summarize, \textit{symbolism} is one of the most influential paradigms in the history of AI~\cite{newell1956logic}. It is rooted in the idea that intelligence can be modeled through the manipulation of symbols---abstract representations of objects, concepts, or relationships in the real world. This paradigm assumes that reasoning and problem-solving can be achieved through symbolic manipulation. \textit{Symbolism} and \textit{connectionism} are highly complementary ~\cite{smolensky1987connectionist}. Incorporating techniques in \textit{symbolism} can effectively solve the intrinsic problems in \textit{connectionism} such as efficiency, robustness, and forgetting. Our work synergizes symbolic is a novel practice in fusing \textit{symbolism} and \textit{connectionism}.

\subsection{Super-expressive Activation Functions}

These days, super-expressive activation functions gain lots of traction in deep learning approximation theory~\cite{savchenko2020exponential}. Unlike the traditional activation functions, super-expressive activation functions can empower a network to achieve an arbitrarily small error with a fixed network width and depth. The idea of super-expressive activation function dates back to \cite{MAIOROV199981}, while the activation function therein was not in closed form. Recently, several studies proposed explicit super-expressive activation functions~\cite{yarotsky2021elementary, 2021Deep, Shen_2021}, such as $\sin(x), \arcsin(x)$, and $x-\mathtt{floor}(x)$ with $x>0$. The key to achieving super-expressiveness lies in leveraging the dense orbit to solve the point-fitting problem $(n,f_n)$, where $f_n\in [0,1], n=1,2,\cdots,N$. The fact that there exists $\theta$ such that $f_n=\theta/(\pi+n)-\mathtt{floor}(\theta/(\pi+n))$ was used in \cite{2021Deep}.

Our work shows that modifying the aggregation function can even endow a network with the super-super-expressive property. Particularly, our theoretical construction is much more efficient in memory. Previously, though the number of parameters like $\theta$ is fixed in \cite{2021Deep}, $\theta$ grows large when the approximation error goes low, and thus the essential network memory is still larger. In contrast, our construction maintains both a fixed number and magnitude of parameters, which brings a significant memory saving. Our work provides a strong theoretical foundation for task-driven neurons. Our proof is inspired by \cite{pmlr-v202-zhang23ad} that also composes a fixed block to do approximation. \cite{pmlr-v202-zhang23ad} uses the bit extract technique, while ours uses the chaotic theory.

\subsection{Task-driven Neuron Based on Vectorized Symbolic Regression}

Symbolic Regression (SR) aims to discover a mathematical expression that best describes a given dataset. Conventional SR methods typically treat each input variable independently, searching for heterogeneous formulas for different features. While flexible, this approach leads to an exponential growth in search space complexity as the input dimension increases. To mitigate this, the concept of Vectorized SR was introduced in the previous work NS-SR~\cite{fan2024no}. The core idea is to impose a homogeneity constraint, enforcing that all input variables share the same underlying transformation rule. This regularization significantly reduces the search space complexity from exponential to linear with respect to the input dimension. As illustrated in Fig.~\ref{fig:preliminary}, frameworks like NS-SR typically follow a two-stage process: i) identifying the optimal homogeneous neuronal formula by fitting task-specific data; 2) implementing the derived formula as the neuron's aggregation function.

However, existing implementations of Vectorized SR predominantly rely on Genetic Programming (GP) as the search engine. While effective in low-dimensional settings, GP-based optimization faces critical bottlenecks in high-dimensional deep learning tasks. The discrete nature of evolutionary search operators, such as crossover and mutation, struggles to traverse the vast combinatorial space efficiently, often leading to premature convergence to local optima. Furthermore, GP-based methods are inherently sensitive to hyperparameters and initialization, resulting in unstable and inconsistent formula discoveries across different runs. These limitations necessitate a more stable, differentiable approach for discovering neuronal formulas, which motivates our proposed NS-TD framework.

\section{Task-driven Neuron Based on Tensor Decomposition (NS-TD)}
\label{sec:nstd}

\subsection{Overview of NS-TD}

To address the instability of GP-based methods, we propose NS-TD, a differentiable framework for neuronal formula discovery. As illustrated in Fig.~\ref{fig:overview}, distinct from the entangled search space in NS-SR, we mathematically decouple the search process into two parallel streams: a \textit{Polynomial Stream} that explicitly models the non-linear contributions of individual features (e.g., polynomial powers), and an \textit{Interaction Stream} dedicated to capturing cross-feature correlations of varying orders via tensor decomposition. This separation ensures that the model can efficiently capture dominant signals without being overwhelmed by the complexity of high-order interactions.

\begin{figure*}[!ht]
    \centering
    \includegraphics[scale=1]{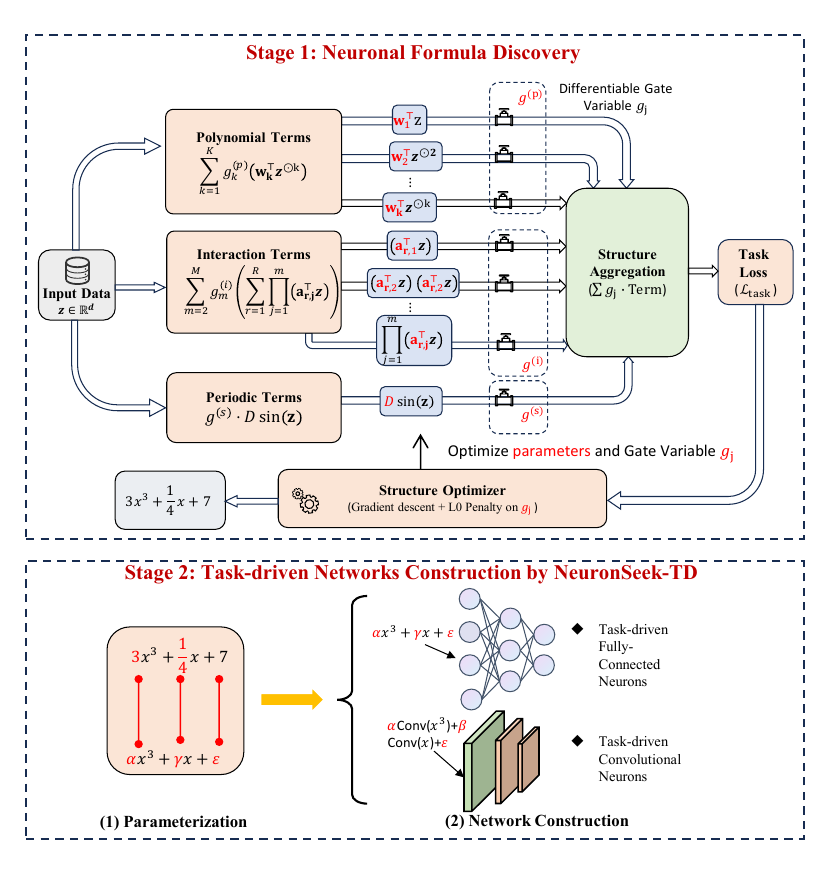}
    \caption{The overall framework of the proposed method. In the first stage, the input data, regardless of tables and images, are flattened into a vector representation and processed using an initial formula for neuronal search. A stable formula is then generated through CP decomposition. In the second stage, the neuronal formula is parameterized and integrated into various neural network backbones for task-specific applications.}
    \label{fig:overview}
\vspace{-0.5cm}
\end{figure*}

\newpage

\subsection{Neuronal Structure Searching}

In the first stage, we aim to discover the optimal mathematical structure for the neuron's aggregation function. Given the input vector $\mathbf{z} \in \mathbb{R}^{d}$, the function $f(\mathbf{z})$ is defined as a combination of polynomial, interaction, and periodic terms, regulated by learnable gates. The general formulation is expressed as
\begin{equation}
\begin{aligned}
    f(\mathbf{z}) &= \underbrace{\sum_{k=1}^{K} g_k^{(p)} \left( \mathbf{w}_k^\top \mathbf{z}^{\odot k} \right)}_{\text{Polynomial Stream}}
    + \underbrace{\sum_{m=2}^{M} g_m^{(i)} \left( \sum_{r=1}^{R} \prod_{j=1}^{m} (\mathbf{a}_{r, j}^\top \mathbf{z}) \right)}_{\text{Interaction Stream}}  + \underbrace{g^{(s)} \cdot D \sin(\mathbf{z})}_{\text{Periodic Term}},
\end{aligned}
\label{eqn:dual_stream_unified}
\end{equation}
where $g^{(\cdot)} \in \{0, 1\}$ denotes the binary gate for term selection. The first component is the Polynomial Stream, where $\mathbf{w}_k^\top \mathbf{z}^{\odot k}$ models the independent $k$-th order polynomial effect, with $\mathbf{z}^{\odot k}$ denoting the element-wise power. The second component corresponds to the Interaction Stream, which utilizes CP decomposition to approximate the $m$-th order feature interactions efficiently. Finally, the periodic term $D\sin(\mathbf{z})$ is included to capture high-frequency patterns.

\subsubsection{Computation of Stream Components}

In this section, we detail the computational mechanisms underlying the polynomial and interaction streams defined above. The computation for the polynomial and periodic streams is straightforward, as they primarily involve element-wise operations. The polynomial term computes $\mathbf{z}^{\odot k}$ by raising each element of the input vector to the power of $k$ independently, while the periodic term applies the sine function to the input vector. These operations maintain a linear computational complexity with respect to the input dimension.

The core computational challenge lies in the interaction stream, which models the high-order coupling between features. A naive implementation of an $m$-th order interaction would require a dense weight tensor $\mathcal{W}^{[m]} \in \mathbb{R}^{d \times \dots \times d}$ with $d^m$ parameters. To mitigate this exponential complexity, we employ the Canonical Polyadic (CP) decomposition. This method factorizes the high-dimensional weight tensor into a summation of rank-1 component tensors. Mathematically, the weight tensor $\mathcal{W}^{[m]}$ is approximated as
\begin{equation}
    \mathcal{W}^{[m]} = \sum_{r=1}^{R} \mathbf{a}_r^{(1)} \circ \mathbf{a}_r^{(2)} \circ \dots \circ \mathbf{a}_r^{(m)},
    \label{eqn:cp_tensor_decomp}
\end{equation}
where $R$ is the decomposition rank, $\circ$ denotes the vector outer product, and $\mathbf{a}_r^{(j)} \in \mathbb{R}^d$ represents the factor vector for the $j$-th mode of the $r$-th rank component. An illustration of a rank-$R$ CP decomposition for a third-order tensor is presented in Fig.~\ref{fig:proofsekcth1}.

\begin{figure}[H]
    \centering
    \includegraphics[width=\linewidth]{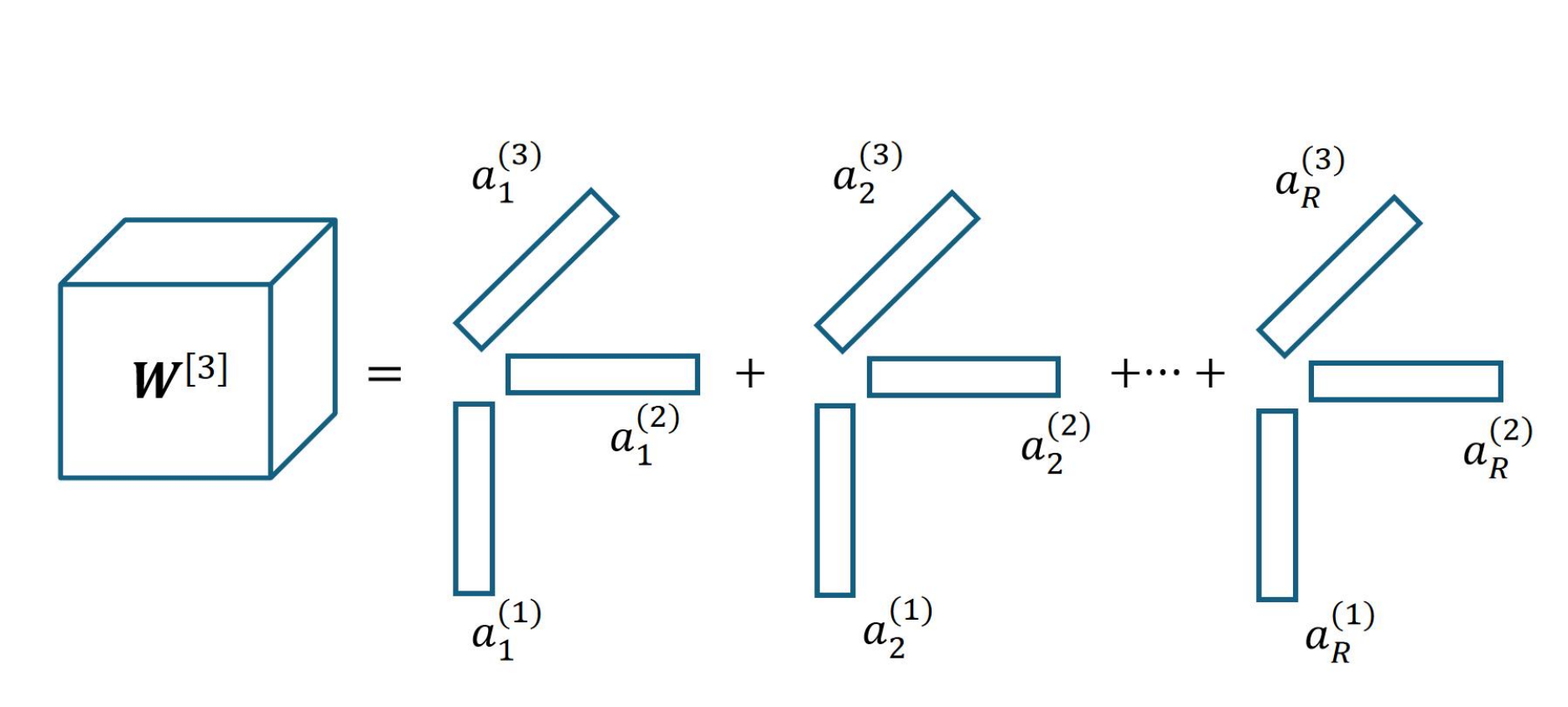}
    \caption{A rank-R CP decomposition of a third-order tensor. }
    \label{fig:proofsekcth1}
\end{figure}

This decomposition provides significant structural and computational advantages. By substituting Eq.~\eqref{eqn:cp_tensor_decomp} into the interaction term, the computation transforms from a high-cost tensor contraction into a series of efficient vector inner products, as formulated in Eq.~\eqref{eqn:dual_stream_unified}. Computationally, this reduces the parameter requirements from $O(d^m)$ to $O(mRd)$. In contrast to other decomposition methods such as Tucker decomposition, which has a complexity of $O(R^m + mRd)$, CP decomposition eliminates the core tensor, thereby ensuring that the model complexity scales linearly with the input dimension $d$. This efficiency allows NS-TD to handle high-order interactions in high-dimensional spaces that are otherwise computationally intractable.

\subsubsection{Differentiable Structure Search}

To effectively identify the optimal structural configuration from the comprehensive formulation in Eq.~\eqref{eqn:dual_stream_unified}, we employ a differentiable search strategy based on the $L_0$ regularization framework proposed by Louizos et al.~\cite{louizos2018learning}. By assigning a stochastic binary gate to each candidate component in the polynomial, interaction, and periodic streams, we transform the discrete term selection problem into a continuous optimization task.

We parameterize the activation probability of each gate using a learnable structural parameter. This formulation enables the gradients to backpropagate directly from the task loss to these parameters, allowing the optimizer to jointly tune the model weights (i.e., $\mathbf{w}$ and $\mathbf{a}$) and the gate probabilities. The global objective function balances predictive accuracy with structural parsimony:
\begin{equation}
    \mathcal{J} = \sum_{(\mathbf{z}, y) \in \mathcal{D}} \mathcal{L}_{task}(f(\mathbf{z}), y) + \lambda \sum_{g} \mathcal{R}(g),
\end{equation}
where $\mathcal{L}_{task}$ is the task-specific loss (e.g., MSE or Cross-Entropy), and the second term aggregates the expected $L_0$ norm penalty $\mathcal{R}(g)$ over all structural gates defined in Eq.~\eqref{eqn:dual_stream_unified}. Through this mechanism, the model automatically suppresses redundant polynomial orders or interaction ranks by driving their associated gate probabilities to zero, effectively performing feature selection during training.

Crucially, to adapt this mechanism to the sensitivity of high-order polynomial and interaction learning, we implement a tailored two-phase optimization protocol. Direct application of sparsity regularization from the onset often leads to the premature pruning of complex terms (e.g., high-rank interactions) before they can capture relevant features. To mitigate this, we initially freeze the structural gates in an open state, focusing the optimization exclusively on training the model tensors to capture the dataset characterization. This warm-up phase ensures that all candidate terms, particularly those in the interaction stream utilizing CP decomposition, develop meaningful representations. In the subsequent phase, we unfreeze the gate parameters and introduce the sparsity regularization, allowing the model to prune terms based on their fully developed representational capacity.

\subsection{Task-driven Neural Network Construction}

Upon identifying the optimal sparse structure via the differentiable search in the first stage, we proceed to construct the task-driven neural network. This phase instantiates the discovered symbolic structure into a learnable neuronal architecture. Specifically, we extract the active index sets $\mathcal{K}$ and $\mathcal{M}$ for the retained polynomial and interaction terms, along with the inclusion status of the periodic term. Distinct from the search phase, here we assign fresh, independent trainable weights to these selected components to maximize representation learning during the final training.

Mathematically, the output $y$ of the constructed neuron is defined as
\begin{equation}
\begin{aligned}
    y = \sigma\Bigg( &\sum_{k \in \mathcal{K}} \mathbf{W}_k^{(p)} (\mathbf{z}^{\odot k}) + \sum_{m \in \mathcal{M}} \mathbf{W}_m^{(i)} \left( \sum_{r=1}^{R} \prod_{j=1}^{m} (\mathbf{a}_{r, j}^\top \mathbf{z}) \right) + \mathbf{W}^{(s)} \sin(\mathbf{z}) + b \Bigg),
\end{aligned}
\end{equation}
where $\sigma(\cdot)$ is the activation function (e.g., ReLU). $\mathbf{W}^{(\cdot)}$ represents the re-initialized trainable weights for the preserved terms. This formulation unifies the processing of tabular and visual data: for tabular inputs, $\mathbf{W}$ indicates scalar coefficients or weight matrices; for image inputs, operations are extended to convolutions.

To ensure training stability for these higher-order architectures, we employ a hierarchical initialization strategy inspired by Chrysos et al.~\cite{chrysos2023regularization}. Parameters corresponding to linear terms (i.e., $k=1$) are initialized using a normal distribution $\mathcal{N}(0, \sigma^2)$ to preserve variance. In contrast, weights for higher-order polynomial terms and interaction terms are initialized with values close to zero (e.g., $10^{-3}$). This approach ensures that the network initially approximates a linear model, gradually integrating non-linear and interactive complexities during the training process. The efficacy of such initialization has been validated in prior studies on quadratic neurons~\cite{fan2023expressivity}.

\section{Super-expressive Property Theory of Task-driven Neurons}

In this section, we formally prove that task-driven neurons that modify the aggregation function can also achieve the super-expressive property. Our theory is established based on the chaotic theory. Mathematically, we have the following theorem:

\begin{theorem}
   Let  $f \in C([0, 1])$  be a continuous function. Then, for any $\varepsilon>0$, there exists a function $h$ generated by a task-driven network with a fixed number of parameters, such that
\begin{equation}
    |f(x)-h(x)|<\varepsilon \quad \text { for any } x \in[0, 1].
\end{equation}
\label{main}
\end{theorem}

\textbf{Proof sketch}: Our proof leverages heavily the framework of \cite{2021Deep} that turns via the interval partition the function approximation into the point-fitting problem. Hence, for consistency and readability, we inherit their notations and framework. The proof idea is divided into three steps:

\textit{Step 1.} As Fig.~\ref{fig:proofsketch} shows, in equal distance, divide $[0,1)$ into small intervals $\mathcal{J}_{k}=\left[\frac{k-1}{K}, \frac{k}{K}\right)$ for $k \in \{1,2, \cdots, K\} $, where $ K $ is the number of intervals. The left endpoint of $\mathcal{J}_{k}$ is $x_{k}$. A higher $K$ leads to a smaller approximation error. We construct a piecewise constant function $h$ to approximate $f$. The error can be arbitrarily small as long as the divided interval goes tiny. Based on the interval division, we have
\begin{equation}
h(x)\approx f(\tilde{x}_k), ~~ x \in \mathcal{J}_{k},
\end{equation}
where $\tilde{x}_k$ is a point from $\mathcal{J}_{k}$.

\textit{Step 2.} As Fig.~\ref{fig:proofsketch} shows, use the function $\phi_{1}(x)=\lfloor K x\rfloor$ mapping all $x$ in the interval $\mathcal{J}_{k}$ to $k$, where $\lfloor\cdot\rfloor$ is the flooring function. There exists a one-to-one correspondence between $k$ and $f(\tilde{x}_k)$. Thus, via dividing intervals and applying the floor function, the problem of approximation is simplified into a point-fitting problem. Therefore, we just need to construct a point-fitting function to map $k$ to $f(\tilde{x}_k)$.

\begin{figure}[!htbp]
    \centering
\includegraphics[width=0.8\linewidth]{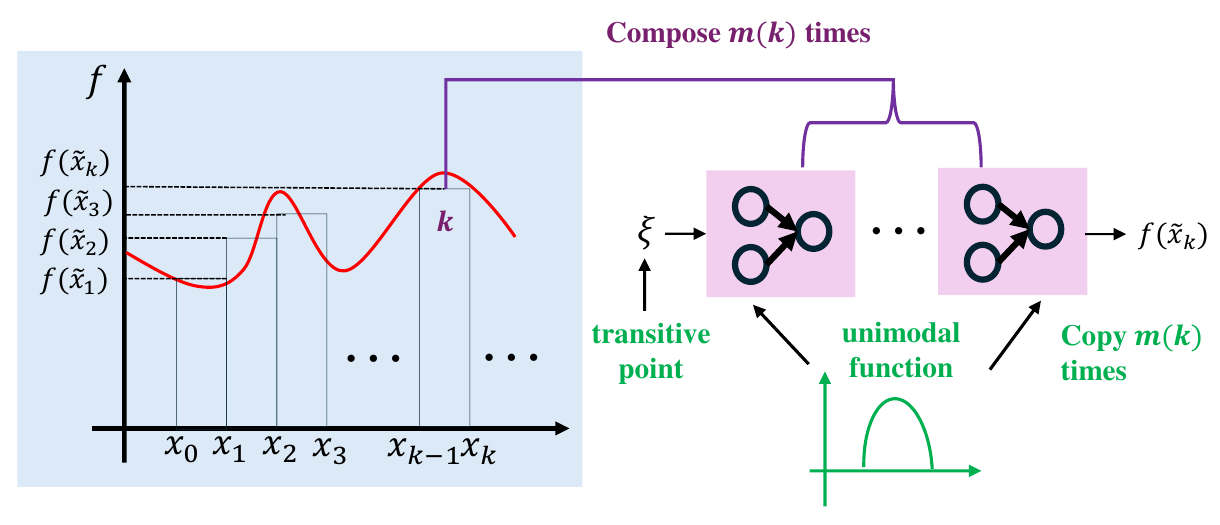}
    \caption{A sketch of our proof. In Steps 1 and 2, through interval partition, the approximation problem is transformed into a point-fitting problem. In Step 3, we use task-driven neurons to construct a unimodal function that can induce the dense trajectory. Then, composing the unimodal function can solve the point-fitting problem.}
    \label{fig:proofsketch}
\end{figure}

\textit{Step 3.} In \cite{2021Deep}, the point-fitting function is constructed as $\tau(\theta/(\pi+n))$ based on Proposition~\ref{irrational_winding}. In contrast, we solve it with a discrete dynamic system that is constructed by a network made of task-driven neurons. We design a sub-network to generate a function $\phi_{2}$ mapping $k$ approximately to  $f\left(\tilde{x}_{k}\right)$ for each $k$. Then $\phi_{2}(\phi_{1}(x))=\phi_{2}(k) \approx f(\tilde{x}_{k}) \approx f(x) $ for any  $x \in \mathcal{J}_{k}$  and $ k \in\{1,2, \cdots, K\} $, which implies $ \phi_{2} \circ \phi_{1} \approx f$ on $[0,1)$. $\phi_{2}$ is $T^{m(k)}(\xi)$ based on Lemma~\ref{eqn:ddt}, where $T$ is a unimodal function from $[0,1] \to [0,1]$ that can also induce the dense trajectory.

\begin{prop}[\cite{katok1995introduction}]
For any given set $\{w_n\}_{n=1}^N \subseteq [0,1]$ and $\epsilon > 0$, there exists a value $\theta^* \in [0,+\infty)$, such that
\begin{equation}
    |w_n-\tau(\frac{\theta^*}{\pi+n})|<\epsilon,\; n=1,\cdots,N,
\label{eqn:core}
\end{equation} where $\tau(z)=z-\lfloor z \rfloor$.
\label{irrational_winding}
\end{prop}

\begin{definition}[\cite{HUANG2005287}]
A map $T: [0,1] \rightarrow [0,1]$ is said to be unimodal if there exists a turning point  $\eta \in [0,1]$  such that the map $f$ can be expressed as
\begin{equation}
    T(x)=\min \left\{T_{\mathrm{L}}(x), T_{\mathrm{R}}(x)\right\}=\left\{\begin{array}{l}
T_{\mathrm{L}}(x), 0 \leqslant x \leqslant \eta, \\
T_{\mathrm{R}}(x), \eta \leqslant x \leqslant 1,
\end{array}\right.
\label{eqn:unimodal_map}
\end{equation}
where $T_{\mathrm{L}}:[0, \eta] \rightarrow[0,1]$  and  $T_{\mathrm{R}}:[\eta, 1] \rightarrow[0,1]$ are continuous, differentiable except possibly at finite points, monotonically increasing and decreasing, respectively, and onto the unit-interval in the sense that $T_{\mathrm{L}}(0)=T_{\mathrm{R}}(1)=0 $ and $ T_{\mathrm{L}}(\eta)=T_{\mathrm{R}}(\eta)=1$.
\end{definition}

\begin{lemma}[Dense Trajectory] There exists a measure-preserving transformation $T$, generated from a network with task-driven neurons, such that there exists $\xi$ with $\{T^n(\xi):n=\mathbb{N}\}$ dense in $[0,1]$. In other words, for any $x$ and $\epsilon>0$, there exist $\xi$ and the corresponding composition times $m$, such that
\begin{equation}
    |x-T^{m} (\xi)|<\epsilon.
\end{equation}
Particularly, we refer to $\xi$ as a transitive point.
\label{eqn:ddt}
\end{lemma}

\begin{proof}
Since we extensively apply polynomials to prototyping task-driven neurons, we prove that a polynomial of any order can fulfill the condition. The proof can be easily extended to other functions with minor twists.

\begin{figure}[!htbp]
    \centering
\includegraphics[width=0.8\linewidth]{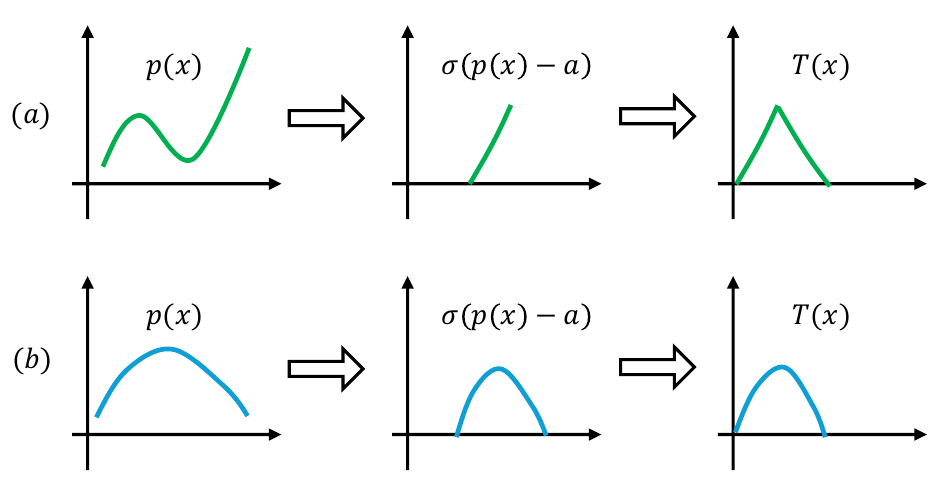}
    \caption{A polynomial can be transformed into a unimodal map.}
  \label{fig:unimodal}
\end{figure}

i) Given a polynomial of any order, as shown in Fig.~\ref{fig:unimodal}, through cutting, translation, flipping, and scaling, we can construct a unimodal map $T(x): [0,1] \to [0,1]$.

ii) For a unimodal map $T$, we can show that for any two non-empty sets $U, V \in X$ with $\mathfrak{m}(U),\mathfrak{m}(V) >0$, there exists an $m$ such that $T^m(U) \cap V \neq \varnothing$. As Fig.~\ref{fig:transitive} shows, this is because a set $U$ with $\mathfrak{m}(U)>0$ will always contain a tiny subset $\mathcal{J}_{mi}$ such that $T^{m}(\mathcal{J}_{mi})=[0,1]$. Then, naturally $T^m(U) \cap V \neq \varnothing$ due to $\mathcal{J}_{mi} \subset U$.

\begin{figure}[!htbp]
    \centering
\includegraphics[width=0.8\linewidth]{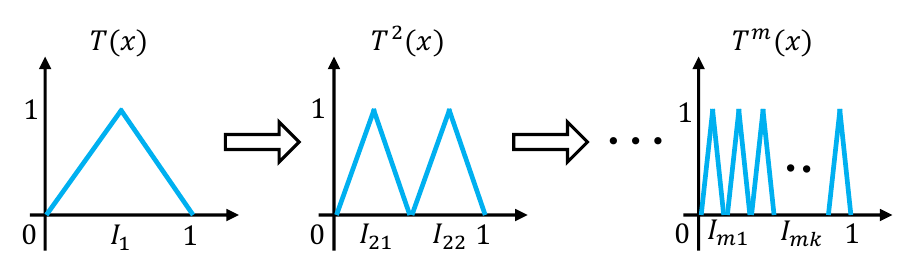}
    \caption{Given a set $U$ with $\mathfrak{m}(U)>0$, it will always contain a tiny subset $\mathcal{I}_{mk}$ such that $T^{mk}(U)=[0,1]$. Then, naturally $T^m(U) \cap V \neq \varnothing$, since $V \in [0,1]$.}
  \label{fig:transitive}
\end{figure}

iii) The above property can lead to a dense orbit, combining that $[0,1]$ is with a countable basis, according to Proposition 2 of \cite{deugirmenci2003existence}. For the sake of self-sufficiency, we include their proof here.

Let $\left(V_{i}\right)_{i \in I}$ be a countable base for  $X$. For $i \in I $, the set  $W_{i}=\bigcup_{n \geqq 0} f^{-n}\left(V_{i}\right)$ is open by continuity of $f$. This set is also dense in $X$. To see this let $ U$  be any non-empty open set. Because of topological transitivity there exists a  $k>0$ with $f^{k}(U) \cap V_{i} \neq \emptyset $. This gives  $f^{-k}\left(V_{i}\right) \cap U   \neq \emptyset$ and $W_{i} \cap U \neq \emptyset$. Thus  $W_{i} $ is dense. By the Baire category theorem, the set  $B=\cap_{i \in I} W_{i} $ is dense in $ X$. Now, the orbit of any point $\xi \in B$ is dense in $X$.

Because, given any non-empty open $ U \subset X $, there is an  $i \in I $ with  $V_{i} \subset U $ and $ k>0$  with $ \xi \in f^{-k}\left(V_{i}\right)$ . This means  $f^{k}(\xi) \in V_{i} \subset U$. Thus the orbit of  $\xi$  enters any $U$.

\end{proof}

\noindent \textit{Proof of Theorem~\ref{main}}.
Now, we are ready to prove Theorem~\ref{main}. First, the input $x$ is mapped into an integer $k$, whereas $k$ corresponds to the composition time $m(k)$. Next, a network module representing $T$ is copied $m(k)$ times and composed $m(k)$ times such that $T^{m(k)}(\xi) \approx f(\tilde{x}_k)$. Since $f(\tilde{x}_k)$ can approximate $f(x)$ well, $T^{m(k)}(\xi)$ also approximates $f(x)$ well.

Theorem~\ref{main} can be easily generalized into multivariate functions by using the Kolmogorov-Arnold theorem~\cite{liu2024kan}.

\begin{prop}
For any $\x=\left[x_{1}, x_{2}, \cdots, x_{d}\right]^{T} \in[0,1]^{d}$, there exist continuous functions  $h_{i, j} \in C([0,1]) $ for  $i=0,1, \cdots, 2 d $ and  $j=1,2, \cdots, d $ such that any continuous function $ f \in C\left([0,1]^{d}\right)$  can be represented as
\begin{equation}
    f(\x)=\sum_{i=0}^{2 d} \Psi_{i}\left(\sum_{j=1}^{d} \Phi_{i, j}\left(x_{j}\right)\right).
\end{equation}
where  $\Psi_{i}: \mathbb{R} \rightarrow \mathbb{R} $ is a continuous function for each  $i \in\{0,1, \cdots, 2 d\} $.
\end{prop}

\textbf{Remark 2}. Our construction reduces the parametric complexity. In \cite{2021Deep}, though the number of parameters remains unchanged, the values of parameters $\theta^*$ increase as long as $n$ increases. Therefore, the overall parametric complexity still goes up, as it consumes more digits. In contrast, in our construction, the parameter value of $T$ is also fixed for different $(k,f(\tilde{x}_k))$. Yet, $T$ is copied and composed $m(k)$ times. The overall parametric complexity remains unchanged. What increases herein is computational complexity. Therefore, we call for brevity our construction ``super-super-expressive''.

\section{Analysis Experiments of NeuronSeek}

\subsection{Experiment on Synthetic Data}\label{sec:syn}

In this section, we compare the performance of our proposed NS-TD with former NS-SR\cite{fan2024no} and traditional symbolic regression methods on synthetic data. We select representative benchmark methods including classic symbolic regression~\cite{schmidt2009distilling}, EQL~\cite{martius2017extrapolation}, and the state-of-the-art method MetaSymNet~\cite{li2025metasymnet}. Synthetic dataset is generated by the functions in Table~\ref{tab:synthetic_formulas}. For each task, input vectors $\mathbf{x} \in \mathbb{R}^d$ are sampled from a standard normal distribution $\mathbf{x} \sim \mathcal{N}(\mathbf{0}, \mathbf{I}_d)$, and the target $\mathbf{y}$ is standardized to zero mean and unit variance. Crucially, regarding the function coefficients, while most are sampled from a normal distribution to ensure diversity, we deliberately assign high-contrast values (e.g., 10 vs. 0.5) in specific cases. This design simulates signal masking scenarios, thereby increasing the difficulty of capturing weak structural information amidst dominant features.

\begin{table}[!htbp]
    \centering
    \caption{\textbf{Synthetic Dataset Generation Formulas.} Comparisons of Pure, Interaction, and Hybrid modes. Coefficients $c_i$ are sampled from $\mathcal{N}(0,1)$.}
    \label{tab:synthetic_formulas}
    
    \small 
    \renewcommand{\arraystretch}{1.2}
    
    \scalebox{1}{
    \begin{tabular}{c | c | c | c}
        \toprule
        \multirow{2}{*}{\textbf{ID}} & \multicolumn{1}{c|}{\textbf{Pure Mode}} & \multicolumn{1}{c|}{\textbf{Interact Mode}} & \multicolumn{1}{c}{\textbf{Hybrid Mode}} \\
         & \multicolumn{1}{c|}{$f(\mathbf{x}) \sim P_k$} & \multicolumn{1}{c|}{$f(\mathbf{x}) \sim I_k$} & \multicolumn{1}{c}{$f(\mathbf{x}) \sim P_k + I_k$} \\
        \midrule
        0 & $c \cdot P_2(\mathbf{x})$ & $c \cdot I_2(\mathbf{x})$ & $c_1 P_2(\mathbf{x}) + c_2 I_2(\mathbf{x})$ \\
        1 & $c \cdot P_3(\mathbf{x})$ & $c_1 I_2(\mathbf{x}) + c_2 I_3(\mathbf{x})$ & $c_1 P_1(\mathbf{x}) + c_2 I_2(\mathbf{x})$ \\
        2 & $10 P_1(\mathbf{x}) + 0.5 P_2(\mathbf{x})$ & $8 I_2(\mathbf{x}) + 0.5 I_3(\mathbf{x})$ & $5 P_3(\mathbf{x}) + 0.5 I_2(\mathbf{x})$ \\
        3 & $5 P_2(\mathbf{x}) + 0.5 P_4(\mathbf{x})$ & $c_1 I_2(\mathbf{x}) + c_2 I_4(\mathbf{x})$ & $0.5 P_2(\mathbf{x}) + 5 I_3(\mathbf{x})$ \\
        4 & $c \cdot P_5(\mathbf{x})$ & $c \cdot I_4(\mathbf{x})$ & $c_1 P_2(\mathbf{x}) + c_2 I_2(\mathbf{x}) + c_3 I_3(\mathbf{x})$ \\
        \bottomrule
    \end{tabular}}
    \parbox{\textwidth}{\footnotesize{\textit{Note:} $P_k(\mathbf{x}) = \sum_{j=1}^d x_j^k$ denotes the element-wise pure polynomial term. $I_k(\mathbf{x}) = \prod_{m=1}^k (\mathbf{w}_m^\top \mathbf{x})$ denotes the interaction term of order $k$, where $\mathbf{w}_m \sim \mathcal{N}(0, 1/\sqrt{d})$.}}
\end{table}

\newpage
The experimental pipeline is structured into two distinct stages: structure discovery and structure evaluation. In the first stage, all methods operate under a strict time budget of 60 seconds to search for the underlying regression formula of a given dataset. Subsequently, we extract the symbolic structure of the discovered formula and proceed to the second stage, where we re-optimize the constants from scratch to evaluate the model's representational capacity. This two stage method is designed to ensure that we solely evaluate the capacity of searched structure. For our proposed NS-TD, we set the tensor rank to 8, which provides sufficient capacity for the synthetic functions, relying on the model's sparsity mechanism to prune redundant components. The regularization coefficient is set to 0.05, a value empirically selected to balance reconstruction fidelity with structural sparsity. For genetic programming-based methods (Standard-SR and TN-SR), following the configuration in~\cite{fan2024no}, we set the population size to 1,000, with crossover, mutation, and reproduction probabilities of 30\%, 60\%, and 10\%, respectively; an elitism strategy is applied to select the top 5\% of individuals for the next generation. For MetaSymNet~\cite{li2025metasymnet}, we adopt an entropy loss regularization of 0.2 and a learning rate of 0.01. The hyperparameter settings for EQL follow those in~\cite{martius2017extrapolation}. For data preparation, we generate 2,500 samples per function, split into training and test sets with an 80/20 ratio.

We report the average MSE across five distinct functions in Table~\ref{tab:synthetic_scalability_main}. NS-TD demonstrates a decisive advantage over competing symbolic regression baselines. Across all functional modes and dimensionalities, our method consistently achieves the lowest test MSE, validating its superior capability in identifying mathematical structures that accurately approximate high-dimensional data manifolds. In contrast, while the tensor network-based NS-SR outperforms traditional discrete search methods like PySR and MetaSymNet in the Pure Mode, it exhibits a significant performance deterioration in the Interact and Hybrid modes. This performance gap stems from NS-SR's architectural reliance on univariate transformations, which precludes the modeling of cross-feature interactions, thereby failing to capture the coupled dynamics inherent in the latter two scenarios. Moreover, NS-TD displays remarkable stability as the input dimension scales from d=10 to d=100. Unlike baseline methods which suffer from convergence difficulties, our approach maintains consistently low error rates, confirming its scalability and robustness in handling structure discovery tasks with high-dimensional data, which is common in deep learning scenarios. Moreover, the Standard Error of the Mean (SEM) across different trials are reported in Table~\ref{tab:synthetic_scalability_appendix1} -~\ref{tab:synthetic_scalability_appendix3} in the 
\textit{Appendix}.

\begin{table*}[!htbp]
    \centering
    \caption{\textbf{Performance Analysis on Synthetic Data.} We evaluate the reconstruction accuracy (Mean MSE) across Pure, Interact, and Hybrid modes under varying input dimensions ($d \in \{10, 30, 50, 100\}$). Best results are highlighted in \textbf{bold}.}
    \label{tab:synthetic_scalability_main}
    \renewcommand{\arraystretch}{1.4}
    \newcolumntype{C}{>{\centering\arraybackslash}p{1.15cm}}
  \scalebox{0.75}{%
    \begin{tabular}{l | C C C C | C C C C | C C C C}
        \toprule
        \multirow{3}{*}{\textbf{Method}} & \multicolumn{4}{c|}{\textbf{Pure Mode}} & \multicolumn{4}{c|}{\textbf{Interact Mode}} & \multicolumn{4}{c}{\textbf{Hybrid Mode}} \\
        & \multicolumn{4}{c|}{{($\sum x^k$)}} & \multicolumn{4}{c|}{{(Vector Interact)}} & \multicolumn{4}{c}{{(Mixture)}} \\
        \cmidrule(lr){2-5} \cmidrule(lr){6-9} \cmidrule(lr){10-13}
        & \textbf{10} & \textbf{30} & \textbf{50} & \textbf{100} & \textbf{10} & \textbf{30} & \textbf{50} & \textbf{100} & \textbf{10} & \textbf{30} & \textbf{50} & \textbf{100} \\
        \midrule
        SR (PySR)
        & 1.1612 & 1.1493 & 1.2325 & 1.2481
        & 0.8754 & 1.0042 & 0.9776 & 0.9763
        & 1.0635 & 0.9821 & 1.0764 & 1.1018 \\
        EQL
        & 0.4524 & 0.5342 & 0.5705 & 0.6923
        & 0.3875 & 0.4862 & 0.5614 & 0.4521
        & 0.3126 & 0.4113 & 0.4237 & 0.4472 \\
        NS-SR (VSR)
        & 0.2324 & 0.2541 & 0.2825 & 0.2953
        & 0.6972 & 0.7734 & 0.7326 & 0.8145
        & 0.7783 & 0.7824 & 0.8281 & 0.8685 \\
        MetaSymNet
        & 0.8324 & 0.8541 & 0.9825 & 0.9843
        & 0.8632 & 1.0804 & 1.0193 & 0.9942
        & 1.1345 & 0.9526 & 1.1093 & 1.1334 \\
        \midrule
        \textbf{NS-TD}
        & \textbf{0.0543} & \textbf{0.0632} & \textbf{0.0854} & \textbf{0.0891}
        & \textbf{0.0512} & \textbf{0.0583} & \textbf{0.0724} & \textbf{0.0945}
        & \textbf{0.0423} & \textbf{0.0621} & \textbf{0.0735} & \textbf{0.1112} \\
        \bottomrule
    \end{tabular}%
    }
\end{table*}

{\textbf{Convergence and Efficiency}}. To further investigate the optimization dynamics, we analyze the convergence trajectory of all methods on the most challenging task (Hybrid Mode, $d = 100$), as illustrated in Fig.~\ref{fig:convergence}. We can observe that our method NS-TD (Red line) converges to the optimal solution faster than evolutionary baselines. This confirms that our differentiable architecture search effectively bypasses the combinatorial search space that hinders methods like NS-SR and SR. Though EQL stabilizes at a plateau at 5.5 seconds, quicker than NS-TD (approximately at 10s), this higher plateau indicates that EQL is stuck in local optima.

\begin{figure}[!htbp]
    \centering
    \includegraphics[width=0.9\linewidth]{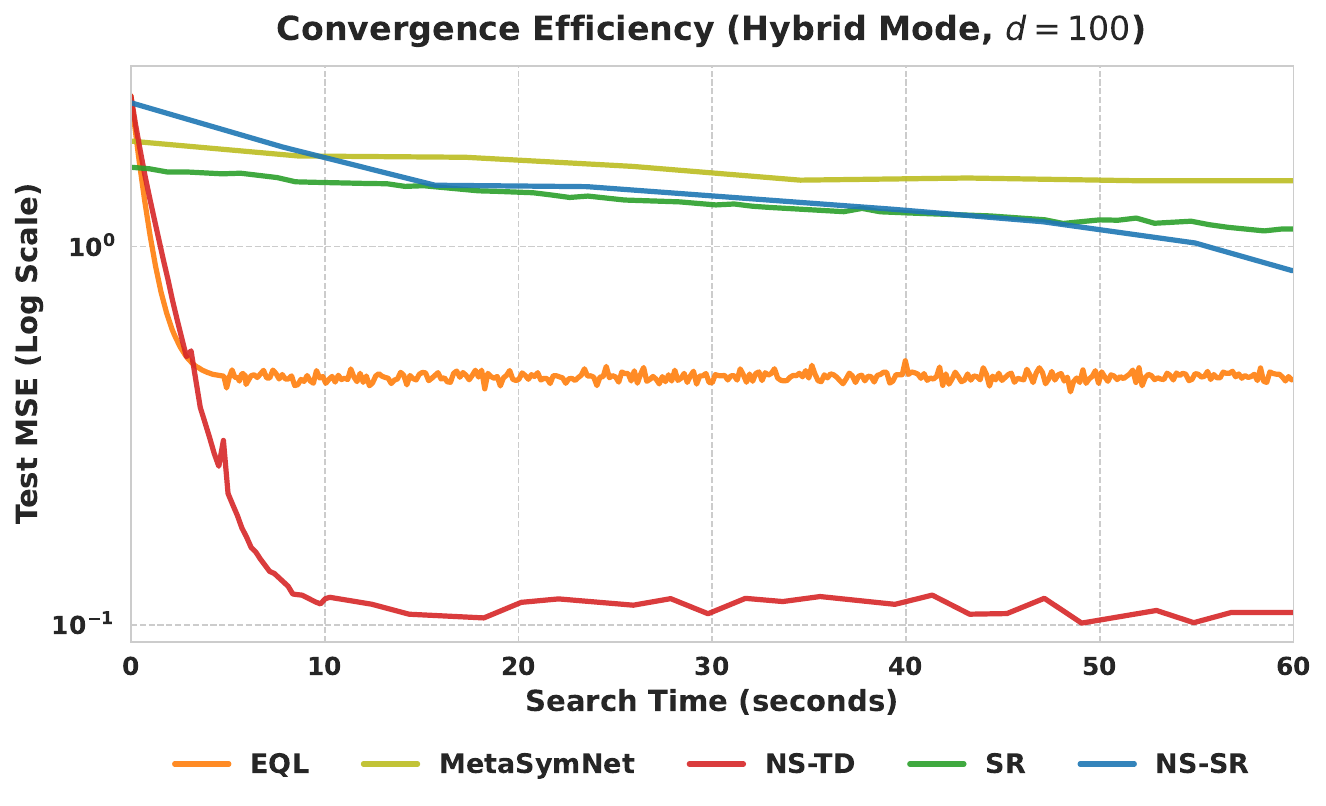}
    \caption{{Convergence Efficiency Analysis (Hybrid Mode, $d=100$).} }
    \label{fig:convergence}
\end{figure}

{\textbf{Why Traditional Symbolic Regression Fails?}} The failure of baseline methods stems from two fundamental limitations. First, the combinatorial explosion in high-dimensional search spaces prevents discrete search algorithms from converging into a good solution within practical time budgets. Second, and more critically, traditional SR tends to generate structurally irregular formulas that are mathematically incompatible with neural architectures that favor homogeneous formulas.

Due to the prohibitive length and structural heterogeneity of baseline outputs, we present a representative case study in Table~\ref{tab:case_study_structure}. As illustrated, baseline methods suffer from consistently observed structural pathologies across high-dimensional tasks: SR generates index-dependent heterogeneous terms; EQL results in excessive density; and MetaSymNet collapses into local features. These heterogeneous structures violate the homogeneity required for tensor operations, meaning they cannot be efficiently vectorized or broadcasted across features. Consequently, unlike the compact forms discovered by NS-TD, these irregular baselines are computationally intractable to serve as scalable aggregation functions for dense deep learning layers.

Deep learning tasks predominantly involve high-dimensional data representations. Through our experiments simulating these high-dimensional generation processes, NS-TD has demonstrated exceptional efficiency in identifying formulas that accurately capture the underlying data patterns. Crucially, in contrast to traditional symbolic regression approaches that yield irregular and computationally intractable expressions, the formulas discovered by NS-TD exhibit the structural homogeneity and compactness required for effective neuronal aggregation. This confirms that NS-TD uniquely satisfies the dual objectives of high reconstruction fidelity and structural suitability for constructing task-based neurons.

\begin{table}[ht]
    \centering
    \caption{\textbf{Case Study: Discovered Structures on Hybrid Mode (Formula ID:0, $d=100$).} We illustrate the symbolic forms identified by different methods. While NS-TD successfully recovers the homogeneous vectorized structure, baselines generate heterogeneous or dense formulas that are ill-suited for neuronal aggregation.}
    \label{tab:case_study_structure}
    \renewcommand{\arraystretch}{1.3}
    \setlength{\tabcolsep}{10pt}
\scalebox{0.8}{
    \begin{tabular}{l | l}
        \toprule
        \textbf{Method} & \textbf{Discovered Structure} \\
        \midrule

        \textbf{Ground Truth}
        & $P_2(\mathbf{x}) + I_2(\mathbf{x})$ \\

        \textbf{NS-TD (Ours)}
        & $P_2(\mathbf{x}) + I_2(\mathbf{x})$ \\

        \midrule

        NS-SR
        & $P_1(\mathbf{x}) + P_2(\mathbf{x})$ \\

        SR
        & $\frac{x_{12} + x_{19}}{x_{16}} + x_{25}^2 - 0.05 x_{12} \dots$ \\

        EQL
        & Mixture of 79 active terms ($\sum w_i \phi_i(\mathbf{x})$) \\

        MetaSymNet
        & $\log(\exp(x_{82})) + x_{45}^2$ \\

        \bottomrule
    \end{tabular}}
\end{table}

\subsection{Experiments on Uniqueness and Stability}

Here, we demonstrate that NS-SR exhibits instability in terms of resulting in divergent formulas under minor perturbations in initialization, while NS-TD maintains good consistency.

\subsubsection{Experimental Settings}
Our experimental protocol consists of two benchmark datasets: the \textit{phoneme} dataset (a classification task classifying nasal and oral sounds) and the \textit{airfoil self-noise} dataset (a regression task modeling self-generated noise of airfoil). These public datasets are chosen from the OpenML website to represent both classification and regression tasks, ensuring a comprehensive evaluation across different problem types. For each dataset, we introduce four levels of Gaussian noise ($\sigma \in \{0.01, 0.025, 0.05, 0.1\}$) to simulate real-world perturbations. We train 10 independent instances of each method under identical initialization. For comparison, NS-SR based on GP uses a population of 3,000 candidates evolved over 20 generations, whereas NS-TD methods---employing CP decomposition and Tucker decomposition---rely on a model optimized via Adam over 20 epochs.

To quantify the uniqueness of SR and TD methods, we track two diversity metrics:
\begin{itemize}
    \item \textbf{Epoch-wise Diversity} represents the number of different formulas within each epoch (or generation) across all initializations. For instance, $x^3 + 2x^2$ and $2x^3+x^2 $ are considered identical, whereas $x^3 + x^2$ and $x^3 + x$ are distinct.
    \item \textbf{Cumulative Diversity} measures the average number of unique formulas identified per initialization as the optimization goes on.
\end{itemize}

\subsubsection{Experimental Results}
The results of \textit{phoneme} and \textit{airfoil self-noise} datasets are depicted in Figs.~\ref{fig:combined_airfoil_diversity} and~\ref{fig:combined_phoneme_diversity}, respectively.

\begin{figure*}[htbp]
    \centering
    \subfloat[Epoch-wise diversity\label{fig:air_epoch_div1}]{%
        \includegraphics[]{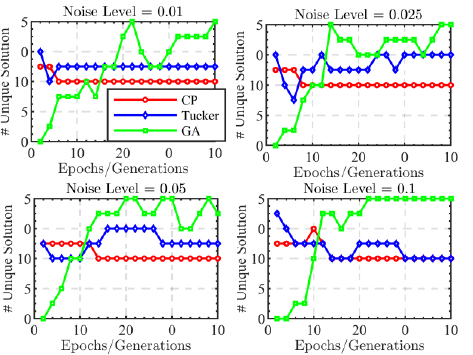}
    }
    \subfloat[Cumulative diversity (average per initialization)\label{fig:unique_air_cumulative}]{%
        \includegraphics[]{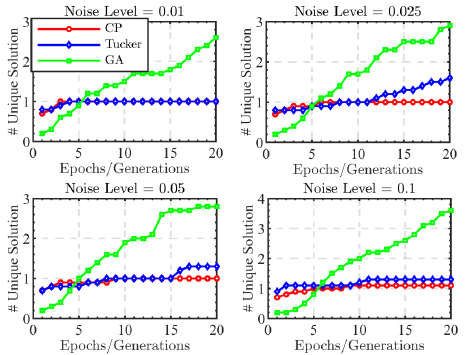}
    }
    \caption{Uniqueness experiment on airfoil self-noise data.}
\label{fig:combined_airfoil_diversity}
\end{figure*}

\begin{figure*}[htbp]
    \centering
    \subfloat[Epoch-wise diversity\label{fig:air_epoch_div2}]{%
        \includegraphics[]{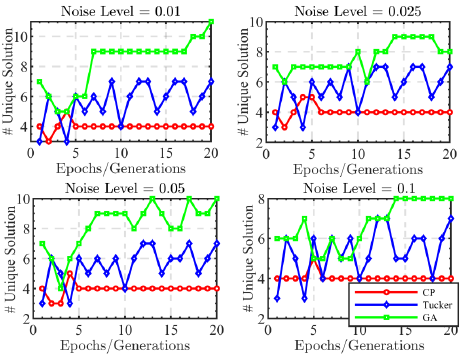}
    }
    \subfloat[Cumulative diversity (average per initialization)\label{fig:unique_phoneme_cumulative}]{%
        \includegraphics[width=0.6\textwidth]{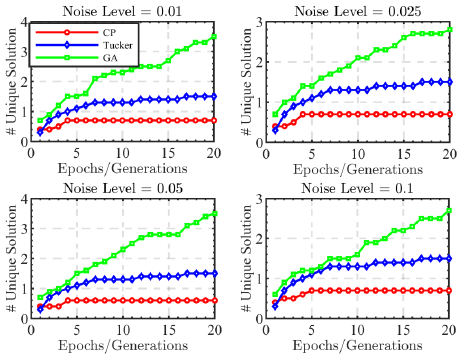}
    }
    \caption{Uniqueness experiment on {phoneme} data.}
    \label{fig:combined_phoneme_diversity}
\end{figure*}

From the epoch-wise diversity (Fig.~\ref{fig:combined_airfoil_diversity} (a) and Fig.~\ref{fig:combined_phoneme_diversity} (a)), we have several highlights: Firstly, GP-based SR demonstrates a rapid increase in the number of discovered formulas within each epoch. After training, it produces more unique formulas compared to TD-based methods (CP and Tucker). This indicates that SR is sensitive to initialization conditions. Secondly, as noise levels increase, the GP method discovers nearly 10 distinct formulas across 10 initializations. With higher noise, every initialization produces a completely different formula structure. These behaviors highlight the instability of GP-based SR when dealing with noisy data, as minor perturbations in initialization or data noise can lead to substantially different symbolic expressions.

Furthermore, regarding cumulative diversity (Fig.~\ref{fig:combined_airfoil_diversity}(b) and Fig.~\ref{fig:combined_phoneme_diversity}(b)), we can observe that GP-based SR constantly discovers an increasing number of different formulas, while TD methods tend to stabilize during training. This indicates that the GP algorithm struggles to converge to a consistent symbolic formula. In contrast, the asymptotic behavior exhibited by TD-based methods demonstrates their superior stability and convergence properties when identifying symbolic expressions. Moreover, the CP-based method suggests the best resilience against initialization variability.

\subsection{Superiority of NS-TD}
\label{sec:TD_exp}

To validate the superior performance of NS-TD, we conduct experiments across 16 benchmark datasets---8 for regression and 8 for classification---selected from scikit-learn and OpenML repositories. This experiment is also conducted in 2 stages. First, we use NS-TD and NS-SR simultaneously to search for the optimal formula. Second, we construct fully connected networks (FCNs) with neurons taking these discovered formulas as their aggregation functions. We evaluate the MSE on regression tasks and accuracy on classification tasks. For comparison, we select Kolmogorov-Arnold Network(KAN)~\cite{liu2024kan} and multi-layer perceptron(MLP) as baselines.

 For regression tasks, the dataset is partitioned into training and test sets in an 8:2 ratio. The activation function is $\texttt{ReLU}$ with MSE as the loss function and $\texttt{RMSProp}$ as the optimizer. For classification tasks, the same train-test split is applied, using the $\texttt{Sigmoid}$ activation function. The loss function and optimizer are set to CE and $\texttt{Adam}$, respectively. All experiments are repeated 10 times and reported by mean. We also report \#parameters and \#FLOPs to show the computation cost of different methods.

The results are illustrated in Table~\ref{tab:full_comparison}, we observe that in regression tasks, NS-TD demonstrates competitive performance, achieving the lowest MSE on four datasets, including \textit{Airfoil}, \textit{Bike Sharing}, \textit{Space GA}, and \textit{Airlines Delay}. Notably, while KAN exhibits strong fitting capability, it suffers from excessive computational overhead (e.g., requiring 10$\times$ more FLOPs than NS-TD on \textit{California Housing}). NS-TD strikes a superior balance, delivering competitive or better accuracy than KAN with significantly lower resource consumption. This advantage is further pronounced in classification tasks, where NS-TD achieves the highest accuracy on 5 out of 8 benchmarks, such as \textit{MagicTelescope}, \textit{Vehicle}, and \textit{Oranges-vs-Grape}. Compared to the MLP baseline, NS-TD consistently yields the performance gain without a proportional increase in model complexity. Collectively, these results confirm that NS-TD effectively generalizes across diverse domains, offering a robust alternative to KAN with greater parameter and computational efficiency. Lastly, we also present the results with standard deviations in the 
\textit{Appendix} Table~\ref{tab:full_comparison_appendix}.

\begin{table}[H]
    \centering
    \caption{\textbf{Comprehensive Comparison on Benchmark Datasets.} We evaluate methods on both regression (MSE $\downarrow$) and classification (Accuracy $\uparrow$) tasks. To assess efficiency, we also report {FLOPs} and {Params}. Note that NS-TD achieves comparable or better performance with significantly lower computational costs.}
    \label{tab:full_comparison}
    \renewcommand{\arraystretch}{1.1}
    \scalebox{0.6}{%
    \begin{tabular}{l | ccc | ccc | ccc | ccc}
        \toprule
        \multirow{2}{*}{\textbf{Dataset}} & \multicolumn{3}{c|}{\textbf{MLP}} & \multicolumn{3}{c|}{\textbf{KAN}} & \multicolumn{3}{c|}{\textbf{NS-SR}} & \multicolumn{3}{c}{\textbf{NS-TD (Ours)}} \\
        & \textbf{Perf.} & \textbf{FLOPs} & \textbf{Param} & \textbf{Perf.} & \textbf{FLOPs} & \textbf{Param} & \textbf{Perf.} & \textbf{FLOPs} & \textbf{Param} & \textbf{Perf.} & \textbf{FLOPs} & \textbf{Param} \\
        \midrule

        \multicolumn{13}{c}{\textit{Regression Tasks (Metric: MSE $\downarrow$)}} \\
        \midrule
        California Housing & \textbf{0.0890} & 1568 & 833 & {0.1017} & 19544 & 7056 & 0.1135 & 3104 & 1601 & \underline{0.1005} & 1568 & 1857 \\
        House Sales        & 0.0077 & 752 & 401 & \textbf{0.0073} & 10551 & 3384 & \underline{0.0075} & 1488 & 769 & \underline{0.0075} & 2224 & 1137 \\
        Airfoil            & 0.0326 & 152 & 89  & \underline{0.0220} & 3017 & 684 & 0.0712 & 152 & 89 & \textbf{0.0209} & 152 & 393 \\
        Diamonds           & 0.0053 & 560 & 305 & \textbf{0.0040} & 8241 & 2520 & \underline{0.0043} & 1104 & 577 & 0.0045 & 560 & 897 \\
        Abalone            & 0.0249 & 200 & 113 & \textbf{0.0232} & 3740 & 900 & 0.0243 & 392 & 209 & \underline{0.0236} & 584 & 305 \\
        Bike Sharing       & 0.0250 & 656 & 353 & 0.0086 & 6460 & 1872 & \underline{0.0081} & 1296 & 673 & \textbf{0.0065} & 1936 & 993 \\
        Space GA           & 0.0064 & 168 & 97  & \textbf{0.0056} & 3258 & 756 & \underline{0.0063} & 328 & 177 & \textbf{0.0056} & 488 & 257 \\
        Airlines Delay     & 0.1605 & 848 & 465 & \underline{0.1563} & 11997 & 3816 & 0.1602 & 1680 & 881 & \textbf{0.1560} & 2512 & 1297 \\

        \midrule
        \multicolumn{13}{c}{\textit{Classification Tasks (Metric: Accuracy $\uparrow$)}} \\
        \midrule
        Credit           & 0.7341 & 608 & 330 & 0.7431 & 3474 & 864 & \underline{0.7441} & 1184 & 618 & \textbf{0.7471} & 1760 & 906 \\
        Heloc            & \underline{0.7107} & 432 & 230 & 0.7027 & 7186 & 1944 & \textbf{0.7117} & 848 & 438 & 0.7077 & 432 & 806 \\
        Electricity      & 0.7762 & 544 & 298 & \textbf{0.7962} & 2992 & 720 & \underline{0.7932} & 1056 & 554 & 0.7862 & 1568 & 810 \\
        Phoneme          & 0.7742 & 448 & 250 & \textbf{0.8292} & 6845 & 2016 & 0.7942 & 864 & 458 & \underline{0.8242} & 448 & 778 \\
        MagicTelescope   & 0.8444 & 608 & 330 & \underline{0.8479} & 8770 & 2736 & 0.8389 & 1184 & 618 & \textbf{0.8495} & 1184 & 1226 \\
        Vehicle          & \underline{0.8106} & 896 & 476 & 0.7576 & 12138 & 4032 & 0.7676 & 1728 & 892 & \textbf{0.8176} & 2560 & 1308 \\
        Oranges-vs-Grape & 0.9105 & 448 & 250 & \underline{0.9302} & 6845 & 2016 & 0.9235 & 864 & 458 & \textbf{0.9344} & 448 & 250 \\
        Eye Movements    & 0.5699 & 400 & 214 & 0.5819 & 6704 & 1800 & \underline{0.5899} & 784 & 406 & \textbf{0.5940} & 400 & 758 \\
        \bottomrule
    \end{tabular}%
    }
\end{table}

\subsection{Comparison of Tensor Decomposition Methods}

Our framework employs the CP decomposition for neuronal discovery. To empirically validate this choice, we conduct a comparative analysis of some tensor decomposition methods used in NS-TD.

\subsubsection{Experimental Settings}

Generally, we follow the previous experimental settings in Section~\ref{sec:TD_exp}: 16 tabular datasets and the same network structures as described in Table~\ref{tab:full_comparison}. We adopt CP decomposition and Tucker decomposition, respectively, for tensor processing in NS-TD. We employ MSE for regression tasks and accuracy for classification tasks.

\subsubsection{Experimental Results}

First, Fig.~\ref{fig:decomp_comparison} illustrates that, for most datasets, the two decomposition methods exhibit comparable performance. Exceptions include \textit{airfoil Self-Noise} and \textit{electricity} datasets, where two methods show slight performance differences. This observation suggests that variations in decompositions for TD methods do not substantially impact the final results. Second, as shown in Fig.~\ref{fig:decomp_time_comparison}, the CP-based method achieves a 42\% reduction in computational time. This result verifies that the CP method offers a computational advantage over the Tucker method.

\begin{figure}[!htbp]
\centering
\subfloat[Regression Task]{
  \includegraphics[width=0.6\columnwidth]{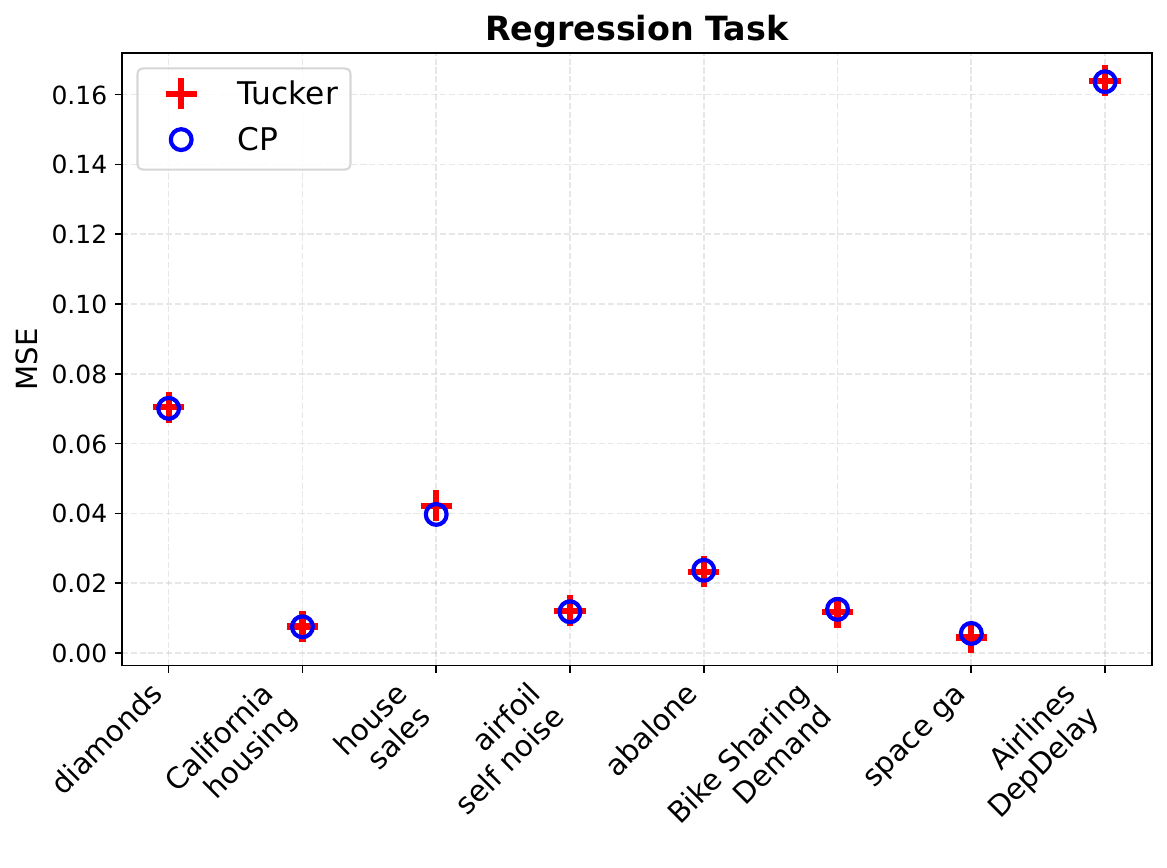}
  \label{fig:decomp_regression}
} 
\subfloat[Classification Task]{
  \includegraphics[width=0.6\columnwidth]{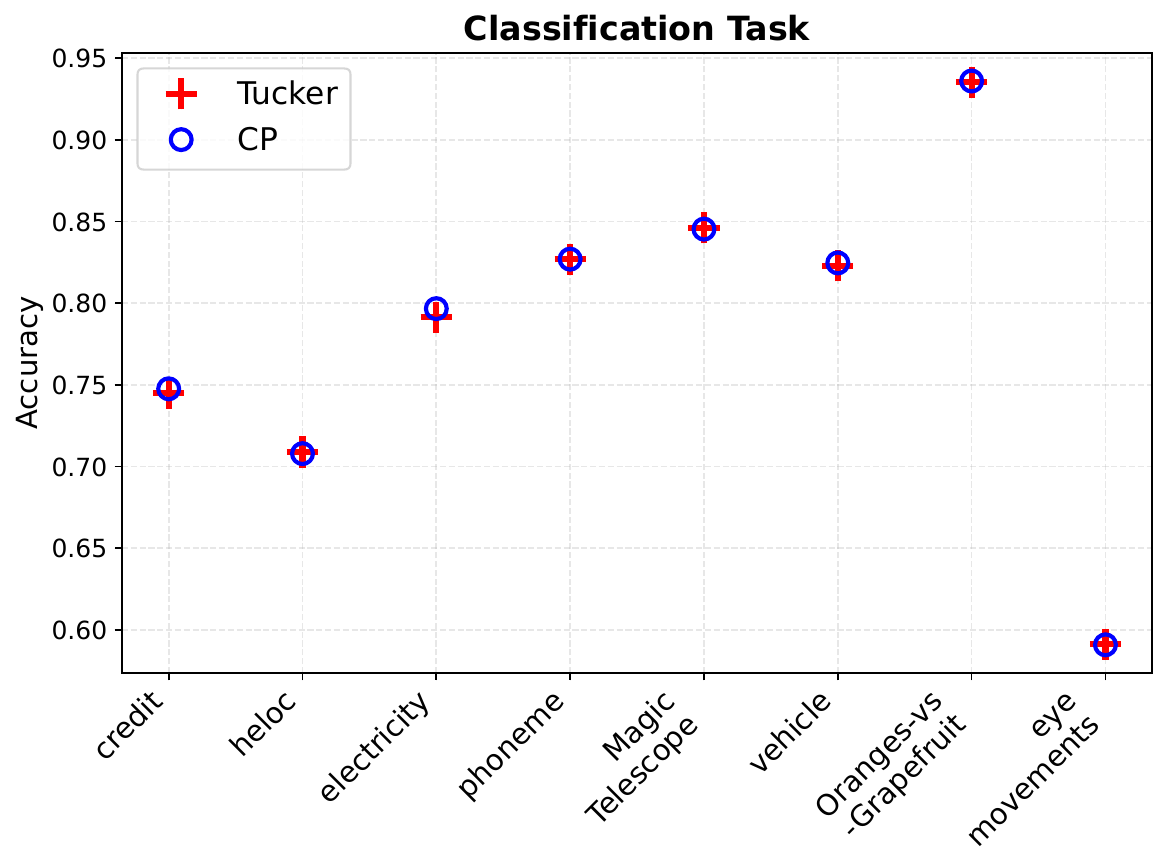}
  \label{fig:decomp_classification}
}
\caption{Comparison between CP and Tucker decomposition.}
\vspace{-0.5cm}
\label{fig:decomp_comparison}
\end{figure}

\begin{figure}[!htbp]
    \centering
    \includegraphics[width=0.6\textwidth]{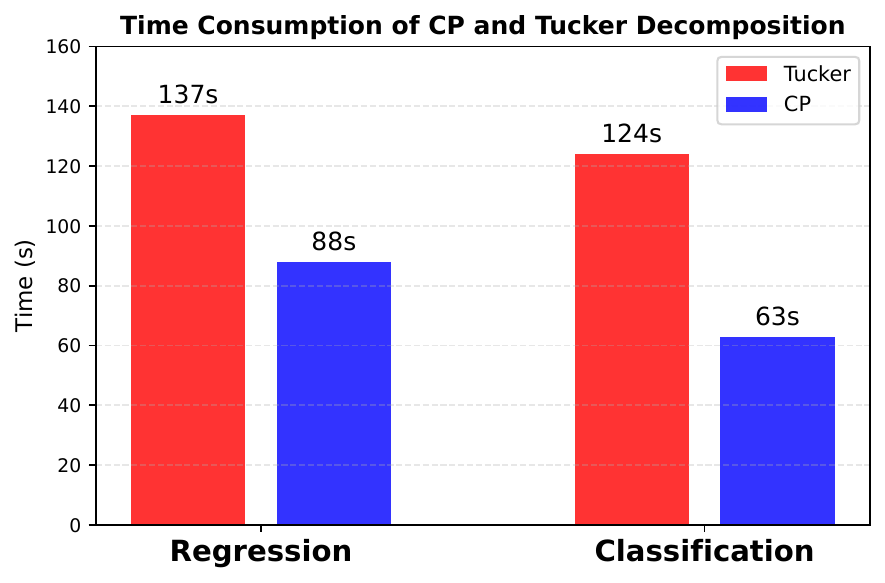}
    \caption{Time comparison between CP and Tucker decomposition.}
    \vspace{-0.5cm}
\label{fig:decomp_time_comparison}
\end{figure}

\section{Competitive Experiments of NS-TD over Other Baselines}

\subsection{Experiment on Tabular Data}

To further illustrate the superiority of NS-TD, we compare it with other state-of-the-art models over 4 real-world tabular tasks: \href{https://www.kaggle.com/datasets/fedesoriano/cern-electron-collision-data}{\textit{Electron Collision Prediction}}, \href{https://www.kaggle.com/datasets/basu369victor/prediction-of-asteroid-diameter}{\textit{Asteroid Prediction}}, \href{https://www.kaggle.com/datasets/fedesoriano/heart-failure-prediction/data}{\textit{Heart Failure Detection}}, and \href{https://www.kaggle.com/datasets/fedesoriano/stellar-classification-dataset-sdss17}{\textit{Stellar Classification}}). We adopt 7 advanced machine learning methods, namely XGBoost~\cite{chen2016xgboost}, LightGBM~\cite{ke2017lightgbm}, CatBoost~\cite{hancock2020catboost}, TabNet~\cite{arik2021tabnet}, TabTransformer~\cite{huang2020tabtransformer}, FT-Transformer~\cite{gorishniy2021revisiting}, DANETs~\cite{chen2022danets} and NS-SR~\cite{fan2024no} as the baselines to compare with NS-TD. Moreover, our proposed model and NS-SR replace standard neurons with what are discovered by NS-TD and NS-SR, respectively.

All datasets are randomly split into training, validation, and test sets in a 0.8:0.1:0.1 ratio. For regression tasks (\textit{electron collision} and \textit{asteroid prediction}), we use the MSE as an evaluation metric, while for classification tasks with imbalanced labels (\textit{heart failure detection} and \textit{stellar classification}), we adopt the F1 score. Detailed experiment information is shown in Table~\ref{tab:real-world data info}.

\begin{table}[htbp]
\centering
 \caption{Details of real-world tabular datasets.}
 \label{tab:real-world data info}
\scalebox{0.8}{\begin{tabular}{lccc}
\hline
\textbf{Dataset}              & \textbf{Task}  & \textbf{Data Size} & \textbf{Features} \\ \hline
\textit{Electron Collision Prediction} & Regression     & 99,915             & 16                \\
\textit{Asteroid Prediction}           & Regression     & 137,636            & 19                \\
\textit{Heart Failure Detection}       & Classification & 1,190              & 11                \\
\textit{Stellar Classification}        & Classification & 100,000            & 17                \\ \hline
\end{tabular}}
\end{table}

In Table~\ref{tab:real-world data comparison}, it can be seen that TabTransformer demonstrates competitive performance on the \textit{electron collision} dataset. However, its effectiveness sharply declines on the asteroid prediction. In contrast, CatBoost and TabNet exhibit more consistent yet suboptimal results across four datasets. Lastly, the proposed NS-TD is the best performer, achieving the lowest error rates and highest F1 scores uniformly across all four benchmarks. Notably, NS-TD significantly outperforms FT-Transformer and DANET. This significant margin underscores the efficacy of its task-adaptive architecture in recognizing feature interactions without domain-specific tuning.

\begin{table*}[!htbp]
    \centering
    \caption{The test results of different models on real-world datasets.}
    \label{tab:real-world data comparison}
    \scalebox{0.65}{%
    \begin{tabular}{lcccc}
        \hline
        \textbf{Method} & \textbf{Electron Collision (MSE)} & \textbf{Asteroid Prediction (MSE)} & \textbf{Heart Failure (F1)} & \textbf{Stellar Classification (F1)}\\
        \hline
        XGBoost & 0.0094 $\pm$ 0.0006 & 0.0646 $\pm$ 0.1031 & 0.8810 $\pm$ 0.02 & 0.9547 $\pm$ 0.002 \\
        LightGBM & 0.0056 $\pm$ 0.0004 & 0.1391 $\pm$ 0.1676 & 0.8812 $\pm$ 0.01 & 0.9656 $\pm$ 0.002 \\
        CatBoost & 0.0028 $\pm$ 0.0002 & 0.0817 $\pm$ 0.0846 & 0.8916 $\pm$ 0.01 & 0.9676 $\pm$ 0.002\\
        TabNet & 0.0040 $\pm$ 0.0006 & 0.0627 $\pm$ 0.0939 & 0.8501 $\pm$ 0.03 & 0.9269 $\pm$ 0.043\\
        TabTransformer & 0.0038 $\pm$ 0.0008 & 0.4219 $\pm$ 0.2776 & 0.8682 $\pm$ 0.02 & 0.9534 $\pm$ 0.002\\
        FT-Transformer & 0.0050 $\pm$ 0.0020 & 0.2136 $\pm$ 0.2189 & 0.8577 $\pm$ 0.02 & 0.9691 $\pm$ 0.002 \\
        DANETs & 0.0076 $\pm$ 0.0009 & 0.1709 $\pm$ 0.1859 & 0.8948 $\pm$ 0.03 & 0.9681 $\pm$ 0.002 \\
        NS-SR & 0.0016 $\pm$ 0.0005 & 0.0513 $\pm$ 0.0551 & 0.8874 $\pm$ 0.05 & 0.9613 $\pm$ 0.002 \\
        NS-TD (ours) & \textbf{0.0011 $\pm$ 0.0003} & \textbf{0.0502 $\pm$ 0.0800} & \textbf{0.9023 $\pm$ 0.03} & \textbf{0.9714 $\pm$ 0.001}\\
        \hline
    \end{tabular}%
    }
\end{table*}

\subsection{Experiments on Images}

To evaluate the performance enhancement of NS-TD in computer vision tasks, we integrate it into several established convolutional neural network architectures: ResNet~\cite{he2016deep}, DenseNet~\cite{huang2017densely}, SeResNet~\cite{hu2018squeeze}, and GoogLeNet~\cite{szegedy2015going}. Notably, we replace the first convolutional layer of each network block with a TN convolutional layer employing the discovered formula $x+\sin x$.  It should be noted that the original NS-SR was not designed for high-dimensional image tasks due to convergence difficulties. Our implementation focuses on NS-TD and follows the standardized training protocol from a popular image classification repository\footnote{\url{https://github.com/weiaicunzai/pytorch-cifar100}}, incorporating several advanced training techniques:  Learning rate scheduling via
$\texttt{WarmUpLR}$ and $\texttt{MultistepLR}$, regularization through $\texttt{Label Smoothing}$, and optimization using $\texttt{SGD}$. We use grid search to fine-tune the learning rate from 0.5 to 0.05. All models are trained for 200 epochs with a batch size of 128. We evaluate the results using classification accuracy averaged over 10 independent runs.

We adopt four standard image datasets to validate the model: \textit{CIFAR 10/100}~\cite{krizhevsky2009learning}, \textit{SVHN}~\cite{netzer2011reading}, and \textit{STL10}~\cite{coates2011analysis}. The training and validation sets are split with a ratio of 0.9:0.1. In particular, images from the \textit{STL10} dataset are resized to $32 \times 32$ to fit the input of the model. The classification results are summarized in Table~\ref{tab:imgclassification}.

\begin{table}[htbp]
\centering
\caption{Accuracy of compared methods on small-scale image datasets.}
\label{tab:imgclassification}
\scalebox{0.8}{
\begin{tabular}{@{}lcccc@{}}
\hline
                       & \textbf{CIFAR10} & \textbf{CIFAR100} & \textbf{SVHN}   & \textbf{STL10}  \\ \hline
ResNet18               & 0.9419           & 0.7385            & 0.9638          & 0.6724          \\
NeuronSeek-ResNet18    & \textbf{0.9464}  & \textbf{0.7582}   & \textbf{0.9655} & \textbf{0.7005} \\ \hline
DenseNet121            & 0.9386           & 0.7504            & 0.9636          & 0.6216          \\
NeuronSeek-DenseNet121 & \textbf{0.9399}  & \textbf{0.7613}   & \textbf{0.9696} & \textbf{0.6638} \\ \hline
SeResNet101            & 0.9336           & 0.7382            & 0.9650          & 0.5583          \\
NeuronSeek-SeResNet101 & \textbf{0.9385}  & \textbf{0.7720}   & \textbf{0.9685} & \textbf{0.6560} \\ \hline
GoogleNet              & 0.9375           & 0.7378            & 0.9644          & 0.7241          \\
NeuronSeek-GoogleNet   & \textbf{0.9400}  & \textbf{0.7519}   & \textbf{0.9658} & \textbf{0.7298} \\ \hline
\end{tabular}
}
\end{table}

Our experiments demonstrate that task-driven neuronal layers consistently enhance the baseline performance across multiple datasets. Notably, NeuronSeek-augmented ResNet18 achieves a 2\% accuracy improvement on both \textit{CIFAR 100} and \textit{STL10} datasets. While SeResNet exhibits median performance on \textit{CIFAR} and \textit{SVHN} datasets, its performance degrades on \textit{STL10}; however, TN layers compensate for this degradation by boosting its accuracy by 10\%. These results suggest that task-driven neuronal formula discovery can effectively unlock the latent potential of fundamental network architectures.

\section{Conclusions}

In this paper, we have proposed a novel method in the framework of NeuronSeek to enhance the capabilities of a neural network. Compared to neural architecture search (NAS)~\cite{white2023neural}, NeuronSeek emphasizes the importance of neurons by searching for optimal aggregation formulas. Firstly, we have introduced tensor decomposition to prototype stable task-driven neurons. Secondly, we have provided theoretical guarantees demonstrating that these neurons can maintain the super-super-expressive property. At last, through extensive empirical evaluation, we have shown that our proposed framework consistently outperforms baseline models while achieving the state-of-the-art results across diverse datasets and tasks. These advancements establish task-driven neurons as an effective addition in the design of task-specific neural network components, offering both theoretical soundness and practical efficacy and stability.

\bibliographystyle{elsarticle-num}

\bibliography{reference}

\clearpage
\appendix
\section{Supplementary Results}

\subsection{Experiments on Synthetic Data}

Supplementary results on synthetic data (Section \ref{sec:syn}) are illustrated as follows.

\begin{table}[!ht]
\centering
\caption{The reconstruction accuracy ($\text{Mean}_{\pm \text{SEM}}$) of pure mode.}
\scalebox{0.9}{
\begin{tabular}{@{}lcccc@{}}
\toprule
\textbf{Method} & \multicolumn{4}{c}{\textbf{Pure Mode}  ($\sum x^k$)}                                                                         \\ \cmidrule(l){2-5} 
                & \textbf{10}                & \textbf{30}                & \textbf{50}                & \textbf{100}               \\ \midrule
SR (PySR)       & $1.1612_{0.0224}$          & $1.1493_{0.0061}$          & $1.2325_{0.0042}$          & $1.2481_{0.0093}$          \\
EQL             & $0.4524_{0.0232}$          & $0.5342_{0.0141}$          & $0.5705_{0.0123}$          & $0.6923_{0.0132}$          \\
NS-SR (VSR)     & $0.2324_{0.0041}$          & $0.2541_{0.0092}$          & $0.2825_{0.0063}$          & $0.2953_{0.0071}$          \\
MetaSymNet      & $0.8324_{0.0642}$          & $0.8541_{0.0053}$          & $0.9825_{0.0081}$          & $0.9843_{0.0092}$          \\ \midrule
\textbf{NS-TD}  & $\mathbf{0.0543}_{0.0062}$ & $\mathbf{0.0632}_{0.0051}$ & $\mathbf{0.0854}_{0.0083}$ & $\mathbf{0.0891}_{0.0082}$ \\ \bottomrule
\end{tabular}}
\label{tab:synthetic_scalability_appendix1}
\end{table}

\begin{table}[!ht]
\centering
\caption{The reconstruction accuracy ($\text{Mean}_{\pm \text{SEM}}$) of interact mode.}
\scalebox{0.9}{
\begin{tabular}{@{}lcccc@{}}
\toprule
\textbf{Method} & \multicolumn{4}{c}{\textbf{Interact Mode} (Vector Interact)}                                                                        \\ \cmidrule(l){2-5} 
                & \textbf{10}                & \textbf{30}                & \textbf{50}                & \textbf{100}               \\ \midrule
SR (PySR)       & $0.8754_{0.0102}$          & $1.0042_{0.0084}$          & $0.9776_{0.0192}$          & $0.9763_{0.0091}$          \\
EQL             & $0.3875_{0.0174}$          & $0.4862_{0.0021}$          & $0.5614_{0.0042}$          & $0.4521_{0.0073}$          \\
NS-SR (VSR)     & $0.6972_{0.0053}$          & $0.7734_{0.0052}$          & $0.7326_{0.0074}$          & $0.8145_{0.0042}$          \\
MetaSymNet      & $0.8632_{0.0064}$          & $1.0804_{0.0093}$          & $1.0193_{0.0102}$          & $0.9942_{0.0073}$          \\ \midrule
\textbf{NS-TD}  & $\mathbf{0.0512}_{0.0041}$ & $\mathbf{0.0583}_{0.0082}$ & $\mathbf{0.0724}_{0.0053}$ & $\mathbf{0.0945}_{0.0042}$ \\ \bottomrule
\end{tabular}}
\label{tab:synthetic_scalability_appendix2}
\end{table}

\begin{table}[!ht]
\centering
\caption{The reconstruction accuracy ($\text{Mean}_{\pm \text{SEM}}$) of hybrid mode.}
\scalebox{0.9}{
\begin{tabular}{@{}lcccc@{}}
\toprule
\textbf{Method} & \multicolumn{4}{c}{\textbf{Hybrid Mode} (Mixture)}                                                                          \\ \cmidrule(l){2-5} 
                & \textbf{10}                & \textbf{30}                & \textbf{50}                & \textbf{100}               \\ \midrule
SR (PySR)       & $1.0635_{0.0083}$          & $0.9821_{0.0072}$          & $1.0764_{0.0123}$          & $1.1018_{0.0241}$          \\
EQL             & $0.3126_{0.0052}$          & $0.4113_{0.0051}$          & $0.4237_{0.0063}$          & $0.4472_{0.0061}$          \\
NS-SR (VSR)     & $0.7783_{0.0081}$          & $0.7824_{0.0072}$          & $0.8281_{0.0031}$          & $0.8685_{0.0083}$          \\
MetaSymNet      & $1.1345_{0.0124}$          & $0.9526_{0.0082}$          & $1.1093_{0.0173}$          & $1.1334_{0.0162}$          \\ \midrule
\textbf{NS-TD}  & $\mathbf{0.0423}_{0.0041}$ & $\mathbf{0.0621}_{0.0043}$ & $\mathbf{0.0735}_{0.0071}$ & $\mathbf{0.1112}_{0.0012}$ \\ \bottomrule
\end{tabular}}
\label{tab:synthetic_scalability_appendix3}
\end{table}

\newpage
\subsection{Experiments on Tabular Data}

We present the supplementary results of Section \ref{sec:TD_exp} to showcase the superiority of the proposed NS-TD.

\begin{table*}[!htbp]
    \centering
    \caption{The supplementary results with standard deviations on the tabular datasets. The metrics are presented in the format $\text{Mean}_{\text{Std}}$.}
    \label{tab:full_comparison_appendix}
\scalebox{1}{
    \begin{tabular}{l | c | c | c | c}
        \toprule
        \textbf{Dataset} & \textbf{MLP} & \textbf{KAN} & \textbf{NS-SR} & \textbf{NS-TD (Ours)} \\
        \midrule

        \multicolumn{5}{c}{\textit{Regression Tasks (Metric: MSE $\downarrow$)}} \\
        \midrule
        California Housing & $\mathbf{0.0890}_{0.0021}$ & $0.1017_{0.0035}$ & $0.1135_{0.0052}$ & $\underline{0.1005}_{0.0050}$ \\
        House Sales        & $0.0077_{0.0005}$ & $\mathbf{0.0073}_{0.0006}$ & $\underline{0.0075}_{0.0008}$ & $\underline{0.0075}_{0.0004}$ \\
        Airfoil            & $0.0326_{0.0055}$ & $\underline{0.0220}_{0.0041}$ & $0.0712_{0.0065}$ & $\mathbf{0.0209}_{0.0072}$ \\
        Diamonds           & $0.0053_{0.0020}$ & $\mathbf{0.0040}_{0.0018}$ & $\underline{0.0043}_{0.0039}$ & $0.0045_{0.0015}$ \\
        Abalone            & $0.0249_{0.0022}$ & $\mathbf{0.0232}_{0.0025}$ & $0.0243_{0.0024}$ & $\underline{0.0236}_{0.0027}$ \\
        Bike Sharing       & $0.0250_{0.0031}$ & $0.0086_{0.0028}$ & $\underline{0.0081}_{0.0026}$ & $\mathbf{0.0065}_{0.0014}$ \\
        Space GA           & $0.0064_{0.0028}$ & $\mathbf{0.0056}_{0.0032}$ & $\underline{0.0063}_{0.0029}$ & $\mathbf{0.0056}_{0.0030}$ \\
        Airlines Delay     & $0.1605_{0.0050}$ & $\underline{0.1563}_{0.0055}$ & $0.1602_{0.0055}$ & $\mathbf{0.1560}_{0.0052}$ \\

        \midrule
        \multicolumn{5}{c}{\textit{Classification Tasks (Metric: Accuracy $\uparrow$)}} \\
        \midrule
        Credit           & $0.7341_{0.0085}$ & $0.7431_{0.0102}$ & $\underline{0.7441}_{0.0092}$ & $\mathbf{0.7471}_{0.0092}$ \\
        Heloc            & $\underline{0.7107}_{0.0120}$ & $0.7027_{0.0135}$ & $\mathbf{0.7117}_{0.0145}$ & $0.7077_{0.0113}$ \\
        Electricity      & $0.7762_{0.0070}$ & $\mathbf{0.7962}_{0.0085}$ & $\underline{0.7932}_{0.0075}$ & $0.7862_{0.0087}$ \\
        Phoneme          & $0.7742_{0.0210}$ & $\mathbf{0.8292}_{0.0245}$ & $0.7942_{0.0207}$ & $\underline{0.8242}_{0.0252}$ \\
        MagicTelescope   & $0.8444_{0.0088}$ & $\underline{0.8479}_{0.0095}$ & $0.8389_{0.0092}$ & $\mathbf{0.8495}_{0.0079}$ \\
        Vehicle          & $\underline{0.8106}_{0.0340}$ & $0.7576_{0.0380}$ & $0.7676_{0.0362}$ & $\mathbf{0.8176}_{0.0354}$ \\
        Oranges-vs-Grape & $0.9105_{0.0150}$ & $\underline{0.9302}_{0.0210}$ & $0.9235_{0.0160}$ & $\mathbf{0.9344}_{0.0240}$ \\
        Eye Movements    & $0.5699_{0.0135}$ & $0.5819_{0.0155}$ & $\underline{0.5899}_{0.0145}$ & $\mathbf{0.5940}_{0.0164}$ \\
        \bottomrule
    \end{tabular}}
\end{table*}

\end{document}